\documentclass{article}
\usepackage[table]{xcolor}
\usepackage{microtype}
\usepackage{graphicx}
\usepackage{subfigure}
\usepackage{booktabs}
\usepackage{multirow}
\usepackage{mathrsfs}
\usepackage{paralist}
\usepackage{comment}
\usepackage{enumitem}
\usepackage{lscape}
\usepackage{afterpage}

\usepackage[hyperfootnotes=false]{hyperref}
\hypersetup{
    colorlinks=true,
    linkcolor=blue,
    filecolor=magenta,      
    urlcolor=cyan,
    pdftitle={ReFactorGNNs},
    pdfpagemode=FullScreen,
}

\definecolor{babyblueeyes}{rgb}{0.63, 0.79, 0.95}
\definecolor{blond}{rgb}{0.98, 0.94, 0.75}
\newcolumntype{s}{>{\columncolor{lightgray}} p{3cm}}

\usepackage[final,nonatbib]{neurips_2022}
\usepackage{amsmath}
\usepackage{amssymb}
\usepackage{mathtools}
\usepackage{amsthm}

\usepackage[capitalize,noabbrev]{cleveref}

\usepackage[utf8]{inputenc} %
\usepackage[T1]{fontenc}    %
\usepackage{url}            %
\usepackage{amsfonts}       %
\usepackage{nicefrac}       %
\usepackage{paralist}
\usepackage{comment}

\DeclareMathOperator{\GD}{GD}
\DeclareMathOperator{\GV}{\bar{v}}
\DeclareMathOperator{\GR}{\bar{r}}
\DeclareMathOperator{\GW}{\bar{w}}
\DeclareMathOperator{\BV}{{\color{blue}v}}
\DeclareMathOperator{\BR}{{\color{blue}r}}

\newcommand{\citep}[1]{\cite{#1}}

\usepackage{wasysym}
\usepackage{marvosym}

\usepackage{xspace}

\newcommand*{\ie}{i.e.\@\xspace}

\newcommand{\factorgnns}{\textsc{ReFactor GNNs}\xspace}
\newcommand{\factorgnn}{\textsc{ReFactor GNN}\xspace}

\theoremstyle{plain}
\newtheorem{theorem}{Theorem}[section]

\newtheorem{lemma}[theorem]{Lemma}

\theoremstyle{definition}

\theoremstyle{remark}

\title{\factorgnns: Revisiting Factorisation-based Models from a Message-Passing Perspective}

\author{
Yihong Chen$^{\text{\Aries \Taurus}}$ \qquad Pushkar Mishra$^{\text{\Taurus}}$ \qquad Luca Franceschi$^{\text{\Gemini}}$\thanks{Work done while at UCL.} 
\qquad Pasquale Minervini$^{\text{\Aries \Libra*}}$ \\ \bf Pontus Stenetorp$^{\text{\Aries}}$ \qquad Sebastian Riedel$^{\text{\Aries \Taurus}}$ \\
$^{\text{\Aries}}$UCL Centre for Artificial Intelligence, London, United Kingdom \\
$^{\text{\Taurus}}$Meta AI, London, United Kingdom \\
$^{\text{\Gemini}}$Amazon Web Services, Berlin, Germany \\
$^{\text{\Libra}}$School of Informatics, University of Edinburgh, Edinburgh, United Kingdom \\
\texttt{\{yihong.chen, p.stenetorp, s.riedel\}@cs.ucl.ac.uk} \\
\texttt{pushkarmishra@meta.com} \quad \texttt{franuluc@amazon.de} \quad \texttt{p.minervini@ed.ac.uk}
}

\begin{document}
\maketitle        
\begin{abstract}
Factorisation-based Models~(FMs), such as DistMult, have enjoyed enduring success for Knowledge Graph Completion~(KGC) tasks, often outperforming Graph Neural Networks~(GNNs).
However, unlike GNNs, FMs struggle to incorporate node features and generalise to unseen nodes in inductive settings.
Our work bridges the gap between FMs and GNNs by proposing \factorgnns.
This new architecture draws upon \textit{both} modelling paradigms, which previously were largely thought of as disjoint.
Concretely, using a message-passing formalism, we show how FMs can be cast as GNNs by reformulating the gradient descent procedure as message-passing operations, which forms the basis of our \factorgnns.
Across a multitude of well-established KGC benchmarks, our \factorgnns achieve comparable transductive performance to FMs, and state-of-the-art inductive performance while using an order of magnitude fewer parameters.
\end{abstract}

\section{Introduction}
In recent years, machine learning on graphs has attracted significant attention due to the abundance of graph-structured data and developments in graph learning algorithms.
Graph Neural Networks~(GNNs) have shown state-of-the-art performance for many graph-related problems, such as node classification~\citep{kipf2016semi} and graph classification~\citep{gilmer2017neural}.
Their main advantage is that they can easily be applied in an inductive setting: generalising to new nodes and graphs without re-training.
However, despite many attempts at applying GNNs for multi-relational link prediction tasks such as Knowledge Graph Completion~\citep{DBLP:conf/aaai/NickelRP16}, there are still few positive results compared to factorisation-based models~(FMs)~\citep{yang2014embedding, pmlr-v48-trouillon16}.
As it stands, GNNs either -- after resolving reproducibility concerns -- deliver significantly lower performance~\citep{nathani2019learning,sun-etal-2020-evaluation} or yield negligible performance gains at the cost of highly sophisticated architecture designs~\citep{Xu2020Dynamically}.
A notable exception is NBFNet~\citep{zhaochengzhu2021}, but even here improvements come at the price of high computational inference costs compared to FMs. 
Furthermore, it is unclear how NBFNet could incorporate node features, which -- as we will see in this work -- leads to remarkably lower performance in an inductive setting.
On the other hand FMs, despite being a simpler architecture, have been found to be very accurate for knowledge graph completion when coupled with appropriate training strategies~\citep{Ruffinelli2020You} and training objectives~\citep{lacroix2018canonical,chen2021relation}.
However, they also come with shortcomings in that they, unlike GNNs, can not be applied in an inductive setting.

Given the respective strengths and weaknesses of FMs and GNNs, \emph{can we bridge these two seemingly different model categories?} 
While exploring this question, we make the following contributions:
\begin{itemize}[leftmargin=*]
    \item{By reformulating gradient descent on node embeddings using message-passing primitives, we show a practical connection between FMs and GNNs, in that: FMs can be treated as a special instance of GNNs, but with infinite neighbourhood, layer-wise training and a global normaliser.\footnote{The traditional view is that \textit{the transductive nature of FMs stem from their need to retrain on new nodes}, a view which we further underpin by also observing that \textit{FMs are not inductive due to the need for infinite layers of on-the-fly message-passing.}}
    }
    \item{Based on this connection, we propose a new family of architectures, referred to as \factorgnns, which interpolates between FMs and GNNs.
    In essence, \factorgnns inductivise FMs by using a finite number of message-passing layers, and incorporating node features.}
    \item{Through an empirical investigation across 15 well-established inductive and transductive benchmarks, we find \factorgnns achieve state-of-the-art inductive performance and comparable transductive performance to FMs -- despite using an order of magnitude fewer parameters.}
\end{itemize}

\section{Background}
\label{sec:kbc}
Knowledge Graph Completion (KGC)~\citep{DBLP:journals/pieee/Nickel0TG16} is a canonical task of multi-relational link prediction. 
The goal is to predict missing edges given the existing edges in the knowledge graph.
Formally, a knowledge graph contains a set of entities~(nodes) $\mathcal{E}=\{1, \ldots, |\mathcal{E}|\}$, a set of relations~(edge types) $\mathcal{R} = \{1, \ldots, |\mathcal{R}|\}$, and a set of typed edges between the entities $\mathcal{T}=\{(v_i, r_i, w_i)\}_{i=1}^{|\mathcal{T}|}$, where each triplet $(v_i, r_i, w_i)$ indicates a relationship of type $r_i \in \mathcal{R}$ between the \textit{subject} $v_i \in \mathcal{E}$ and the \textit{object} $w_i \in \mathcal{E}$. 
Given a node $v$, we denote its \textit{outgoing} 1-hop neighbourhood as the set of relation-object pairs $\mathcal{N}^1_+[\BV] = \{(r, o) \mid (\BV, r, o) \in \mathcal{T} \}$, its \textit{incoming} 1-hop neighbourhood as the set of subject-relation pairs 
$\mathcal{N}^1_-[\BV] = \{(r, s) \mid (s, r, \BV) \in \mathcal{T} \}$, and $\mathcal{N}^1[v] = \mathcal{N}^1_+[v] \cup \mathcal{N}^1_-[v]$ the union of the two.
We denote the neighbourhood of $v$ under a specific relation $r$ as $\mathcal{N}_{\pm}^1[r, v]$.
Entities may come with features $X\in \mathbb{R}^{|\mathcal{E}| \times K}$ for describing them, such as textual encodings of their names and/or descriptions.
Given a~(training) knowledge graph, KGC is evaluated by answering $(v, r, ?)$-style queries \ie predicting the object given the subject and relation in the triplet. 
And queries like $(?, r, v')$ are answered using the inverse queries $(v', r^{-1}, ?)$ in this work, following~\citep{lacroix2018canonical}.

Multi-relational link prediction models can be trained via maximum likelihood, by fitting a parameterized conditional categorical distribution $P_{\theta}(w \mid v, r)$ over the candidate objects of a relation, given the subject $v$ and the relation type $r$:
\begin{equation}
P_{\theta}(w|\BV, \BR) 
= \frac{ \exp{\Gamma_{\theta}(\BV, \BR, w)} }
       { \sum_{u \in \mathcal{E}} \exp{\Gamma_{\theta}(\BV, \BR, u)} } 
= \mathrm{Softmax}(\Gamma_{\theta}(\BV, \BR, \cdot))[w].
\end{equation}
Here $\Gamma_{\theta}: \mathcal{E} \times \mathcal{R} \times \mathcal{E} \to \mathbb{R}$ is a \textit{scoring function}, which, given a triplet $(v, r, w)$, returns the likelihood that the corresponding edge appears in the knowledge graph.
We illustrate our derivations  using DistMult~\cite{yang2014embedding} as the score function $\Gamma$ and defer extensions to general score functions, e.g. ComplEx~\cite{pmlr-v48-trouillon16}, to the appendix. 
For DistMult, the score function $\Gamma_{\theta}$ is defined as the tri-linear dot product of the vector representations corresponding to the subject, relation, and object of the triplet:
\begin{equation}
\label{eq:distmult}
\Gamma_{\theta}(v,r,w) = \langle f_{\phi}(v), f_{\phi}(w), g_{\psi}(r) \rangle = \sum_{i = 1}^{K} f_{\phi}(v)_{i} f_{\phi}(w)_{i} g_{\psi}(r)_{i},
\end{equation}
where $f_{\phi}: \mathcal{E} \to \mathbb{R}^K$ and $g_{\psi}:\mathcal{R}\to\mathbb{R}^K$ are learnable maps parameterised by $\phi$ and $\psi$ that encode entities and relation types into $K$-dimensional vector representations, and $\theta=(\phi, \psi)$. We will refer to $f$ and $g$ as the entity and relation \emph{encoders}, respectively.
If we define the data distribution as $P_D(x) = \frac{1}{|\mathcal{T}|} \sum_{(v, r, w) \in \mathcal{T}} \delta_{(v,r,w)}(x)$, where $\delta_{(v,r,w)}(x)$ is a Dirac delta function at $(v,r,w)$, then the objective is to learn the model parameters $\theta$ by minimising the expected negative log-likelihood $\mathcal{L}(\theta)$ of the ground-truth entities for the queries $(v, r, ?)$ obtained from $\mathcal{T}$:
\begin{equation} 
\label{eq:objective}
\arg\min_{\theta} \mathcal{L}(\theta) \quad \text{where} \quad
\mathcal{L}(\theta) = - \mathbb{E}_{x \sim P_D} [\log (P_{\theta}(w|v, r)] =
- \frac{1}{|\mathcal{T}|} \sum_{(v, r, w) \in \mathcal{T}}  \log P_{\theta}(w|v, r).
\end{equation}
During inference, we use $P_{\theta}$ for determining the plausibility of links not present in the training graph. 
\subsection{Factorisation-based Models for KGC}
\label{sec:bg_fm}
In factorisation-based models, which we assume to be DistMult, $f_{\phi}$ and $g_{\psi}$ are simply parameterised as look-up tables, associating each entity and relation with a continuous distributed representation:
\begin{equation}
\label{eq:fm_entity_encoder}
f_{\phi}(v) = \phi[v], \; \phi \in \mathbb{R}^{|\mathcal{E}| \times K} \quad \text{and} \quad 
g_{\psi}(r) = \psi[r], \; \psi \in \mathbb{R}^{|\mathcal{R}| \times K}.
\end{equation}

\subsection{GNN-based Models for KGC}
\label{sec:gnn_kbc}
GNNs were originally proposed for node or graph classification tasks~\citep{Gori2005ANM,DBLP:journals/tnn/ScarselliGTHM09}.
To adapt them to KGC, previous work has explored two different paradigms: \emph{node-wise entity representations} \cite{schlichtkrull2018modeling} and \emph{pair-wise entity representations}~\cite{Teru2020InductiveRP,zhaochengzhu2021}.
Though the latter paradigm has shown promising results, it requires computing representations for all pairs of nodes, which can be computationally expensive for large-scale graphs with millions of entities.
Additionally, node-wise representations allow for using a single evaluation of $f_{\phi}(v)$ for multiple queries involving $v$, resulting in fast batch evaluation.

Models based on the first paradigm differ from pure FMs only in the entity encoder and lend themselves well for a fairer comparison with pure FMs.
We will therefore focus on this class and leave the investigation of pair-wise representations to future work.
Let $q_{\phi}:\mathcal{G}\times\mathcal{X} \to \bigcup_{S\in\mathbb{N}^+} \mathbb{R}^{S\times K}$ be a GNN encoder, where
$\mathcal{G} = \{ G \mid G \subseteq \mathcal{E} \times \mathcal{R} \times \mathcal{E} \}$ is the set of all possible multi-relational graphs defined over $\mathcal{E}$ and $\mathcal{R}$, and $\mathcal{X}$ is theinput feature space, respectively.
Then we can set  $f_\phi(v)=q_{\phi}(\mathcal{T}, X)[v]$.
Following the standard message-passing framework ~\cite{gilmer2017neural,battaglia2018relational,grlbook} used by the GNNs, we view $q_{\phi}=q^L \circ ... \circ q^1$ as the recursive composition of $L \in \mathbb{N}^+$ layers that compute intermediate representations $\{h^l\}$ 
for $l\in \{1, \dots, L\}$  
(and $h^0=X$) for all entities in the KG. 
Each layer $q^l$ producing representation $h_l$ is made up of the following three functions: 
\begin{enumerate}
\item A \textit{message function} $q^l_{\mathrm{M}}: \mathbb{R}^K \times \mathcal{R} \times \mathbb{R}^K \to \mathbb{R}^K$ that computes the message along each edge. Given an edge $(v, r, w)\in\mathcal{T}$, $q^l_{\mathrm{M}}$ not only makes use of the node states $h^{l-1}[v]$ and $h^{l-1}[w]$~(as in standard GNNs) but also uses the relation $r$; denote the message as $$m^l[v, r, w] = q^l_{\mathrm{M}}\left(h^{l-1}[v], r, h^{l-1}[w]\right);$$
    \item An \textit{aggregation function} $q^l_{\mathrm{A}}: \bigcup_{S\in\mathbb{N}}\mathbb{R}^{S \times K} \to \mathbb{R}^K$
    that aggregates all messages from the 1-hop neighbourhood of a node; denote the aggregated message as 
$$z^l[v] = q^l_{\mathrm{A}} \left(\{m^l[v, r, w]\; |\; (r,w) \in \mathcal{N}^1[v]\} \right);$$   
\item An \textit{update function} $q^l_{\mathrm{U}}: \mathbb{R}^K \times \mathbb{R}^K \to \mathbb{R}^K$ that produces the new node states $h^l$ by combining previous node states $h^{l-1}$ and the aggregated messages $z^l$: $$h^l[v] = q^l_{\mathrm{U}}(h^{l-1}[v], z^{l}[v]).$$
\end{enumerate}
Different parameterisations of $q^l_{\mathrm{M}}$, $q^l_{\mathrm{A}}$, and $q^l_{\mathrm{U}}$ lead to different GNNs. 
For example, R-GCNs~\citep{schlichtkrull2018modeling} define the $q^l_{\mathrm{M}}$ function using per-relation linear transformations 
$
m^l[v, r, w]= \frac{1}{\mathcal{N}^1[r, v]}W_{r}^lh^{l-1}[w]
$;
$q^l_{\mathrm{A}}$ is implemented by a summation and  $q^l_{\mathrm{U}}$ is a non-linear transformation 
$
h^l[v] = \sigma(z^l[v] + W_0^l h^{l-1}[v])
$, 
where $\sigma$ is the sigmoid function.
For each layer, the learnable parameters are $\{W_{r}^l\}_{r\in\mathcal{R}}$ and $W_0^l$, all of which are matrices in $\mathbb{R}^{K\times K}$.
Sometimes applying GNNs over an entire graph might not be feasible due to the size of the graph. Hence, in practice, $f_{\phi}(v)$ can be approximated with sampled sub-graphs  ~\cite{hamilton2017inductive,ladies2019,graphsaint-iclr20}, such as $L$-hop neighbourhood around node $v$ denoted as $\mathcal{N}^L[v]$:
\begin{equation}
\label{eq:gnn_sub-graph}
f_\phi(v) = q_\phi(\mathcal{T}_{\mathcal{N}^L[v]}, X_{\mathcal{N}^L[v]}) [v].
\end{equation}

\section{Implicit Message-Passing in FMs}
\label{sec:fm_sgd}
The sharp difference in analytical forms might give rise to the misconception that GNNs incorporate message-passing over the neighbourhood of each node~(up to $L$-hops), while FMs do not. 
In this work, we show that by explicitly considering the training dynamics of FMs, we can uncover and analyse the hidden message-passing mechanism within FMs.
In turn, this will lead us to the formulation of a novel class of GNNs well suited for multi-relational link prediction tasks (\cref{sec:refactorgnn}).  
Specifically, we propose to interpret the FMs' optimisation process of their objective \eqref{eq:objective} as the entity encoder.
If we consider, for simplicity, a gradient descent training dynamic, then
\begin{equation}
\label{eq:GD_layers}
f_{\phi^t}(v) = \phi^t[v] 
= \GD^t(\phi^{t-1}, \mathcal{T})[v]
= \underbrace{\GD^t \circ ... \GD^1}_{t}(\phi^0, \mathcal{T})[v],
\end{equation}
where $\phi^{t}$ is the embedding vector at the  $t$-th step,   $t\in\mathbb{N}^+$ is the total number of iterations and $\phi^0$ is a random initialisation.
$\GD$ is the gradient descent operator:
\begin{equation}
\label{eq:theta_gd}
\GD(\phi, \mathcal{T}) = \phi - \alpha \nabla_{\phi}\mathcal{L} = \phi + \alpha \sum_{(v, r, w) \in \mathcal{T}} \frac{\partial \log P(w|v, r)}{\partial \phi},
\end{equation}
where $\alpha=\beta\,|\mathcal{T}|^{-1}$, with a $\eta>0$ learning rate.
We now dissect \cref{eq:theta_gd} in two different~(but equivalent) ways.
In the first, which we dub the \textit{edge view}, we separately consider each addend of the gradient $\nabla_{\phi}\mathcal{L}$.
In the second, we aggregate the contributions from all the triplets to the update of a particular node. 
With this latter decomposition, which we call the \textit{node view}, we can explicate the message-passing mechanism at the core of the FMs. 
While the edge view suits a vectorised implementation better, the node view further exposes the information flow among nodes, allowing us to draw an analogy to message-passing GNNs. 
\subsection{The Edge View}
\label{sec:edge_view}
Each addend of \cref{eq:theta_gd} corresponds to a single edge $(v, r, w)\in\mathcal{T}$ and contributes to the update of the representation of all nodes.
The update on the representation  of the subject $\phi[v]$ is:
\begin{align*}
 \GD(\phi, \{(v, r, w)\})[v] = \phi[v] + \alpha\left(
 \textcolor{black}{
    \underbrace{
    g(r)\odot\phi[w]}_{w \to v}} 
 \textcolor{black}{
    \underbrace{- \sum_{u\in \mathcal{E}}P_\theta(u|v,r)g(r)\odot\phi[u]}_{u \to v}}
 \right).
\end{align*}
The $w \to v$ term indicates information flow from $w$~(a neighbour of $v$) to $v$, increasing the score of the gold triplet $(v, r, w)$.
The $u \to v$ term indicates information flow from global nodes, decreasing the scores of triplets $(v, r, ?)$ with $v$ as the subject and $r$ as the predicate.
Similarly, for the object $w$,
\begin{align*}
\GD(\phi, \{(v, r, w)\})[w]
= \phi[w] + \alpha \textcolor{black}{\underbrace{\left(1 - P_\theta(w|v,r)\right) g(r)\odot\phi[v]}_{v \to w}},
\end{align*}
where, again, the $v \to w$ term indicates information flow from the neighbouring node $v$.
Finally, for the nodes other than $v$ and $w$, we have 
\begin{align*}
\GD(\phi, \{(v, r, w)\})[u] = \phi[u] + \alpha \left(\textcolor{black}{\underbrace{-P_\theta(u|v,r)\phi[v] \odot g(r)}_{v \to u}}
\right). \\
\end{align*}
\subsection{The Node View}
To fully uncover the message-passing mechanism of FMs, we now focus on the gradient descent operation over a single node $v\in\mathcal{E}$, referred to as the \textit{central node} in the GNN literature.
Recalling \cref{eq:theta_gd}, we have:
\begin{equation}
\label{eq:fm_node_view}
\GD(\phi, \mathcal{T})[v] = \phi[v] + \alpha\sum_{(v, r, w) \in \mathcal{T}}\frac{\partial\log P(\GW|\GV, \GR)}{\partial\phi[v]},  
\end{equation}
which aggregates the information stemming from the updates presented in the edge view.
The next theorem describes how this total information flow to a particular node can be recast as an instance of message passing (cf. \cref{sec:gnn_kbc}). We defer the proof to the appendix.

\begin{theorem}[Message passing in FMs] 
\label{the:mp_fm}
The gradient descent operator $\GD$ (\cref{eq:fm_node_view}) on the node embeddings of a DistMult model (\cref{eq:fm_entity_encoder}) with the maximum likelihood objective in \cref{eq:objective} and a multi-relational graph $\mathcal{T}$ defined over entities $\mathcal{E}$ induces a message-passing operator whose composing functions are:  
\begin{align} 
 &q_{\mathrm{M}}(\phi[v], r, \phi[w]) = 
 \left\lbrace
\begin{array}{lr}
\phi[w] \odot  g(r) & \text{if} \; (r,w) \in \mathcal{N}_{+}^1[v], \\
(1 - P_\theta (v|w, r)) 
\phi[w] \odot g(r)  & \text{if} \; (r, w) \in \mathcal{N}_-^1[v];
\end{array}
\right.
\label{eq:fm_qm}
\\
&q_{\mathrm{A}}(\{m[v, r, w]\, :\, (r,w) \in \mathcal{N}^1[v]\}) = \sum_{(r,w) \in \mathcal{N}^1[v]} m[v,r,w]; \\
&q_{\mathrm{U}}(\phi[v], z[v]) = \phi[v] + \alpha z[v] - \beta 
n[v],
\label{eq:fm_qu}
\end{align}
where, defining the sets of triplets $\mathcal{T}^{-v} =
\{
(s, r, o) \in \mathcal{T} \; : \; s\neq v \wedge o\neq v
\}
$, 
\begin{equation}
n[v]= 
\frac{|\mathcal{N}_{+}^{1}[v]|}{|\mathcal{T}|}
\mathbb{E}_{
P_{\mathcal{N}_+^{1}[v]}} \mathbb{E}_{u \sim P_{\theta}(\cdot|v, r)} g(r) \odot \phi[u] 
+
\frac{|\mathcal{T}^{-v}|}{|\mathcal{T}|}
\mathbb{E}_{ P_{\mathcal{T}^{-v}}}P_\theta(v|s, r) g(r) \odot \phi[s] ,
\label{eq:fm_normaliser}
\end{equation}
where $P_{\mathcal{N}^{1}_+[v]}$ and $P_{\mathcal{T}^{-v}}$ are the empirical probability distributions associated to the respective sets.
\end{theorem}
What emerges from the equations is that each $\GD$ step contains an explicit information flow from the neighbourhood of each node, which is then aggregated with a simple summation.
Through this direct information path, $t$ $\GD$ steps cover the $t$-hop neighbourhood of $v$.
As $t$ goes towards infinity -- or in practice -- as training converges, FMs capture the global graph structure.
The update function \eqref{eq:fm_qu} somewhat deviates from classic message passing frameworks as $n[v]$ of  \cref{eq:fm_normaliser} involves global information. 
However, we note that we can interpret this mechanism under the framework of augmented message passing \cite{augmented-mp} and, in particular, as an instance of \textit{graph rewiring}.

Based on \cref{the:mp_fm} and \cref{eq:GD_layers}, we can now view $\phi$ as the transient node states $h$~(cf. \cref{sec:gnn_kbc}) and $\GD$ on node embeddings as a message-passing layer.
This dualism sits at the core of the ReFactor GNN model, which we describe next.

\section{\textsc{ReFactor} GNNs}
\label{sec:refactorgnn}
FMs are trained by minimising the objective \eqref{eq:objective}, initialising both sets of parameters~($\phi$ and $\psi$) and performing GD until approximate convergence (or until early stopping terminates the training). 
 The implications are twofold: $i$) the initial value of the entity lookup table $\phi$ does not play any major role in the final model after convergence, and $ii$)  if we introduce a new set of entities, the conventional wisdom is to retrain\footnote{
 Typically until convergence, possibly by partially  warm-starting $\theta$.
 } the model on the expanded knowledge graph. 
 This is computationally rather expensive compared to the ``inductive'' models that require no additional training and can leverage node features like entity descriptions.
However, as we have just seen in~\cref{the:mp_fm}, the training procedure of FMs may be naturally recast as a message-passing operation, which suggests that it is possible to use FMs for inductive learning tasks. 
In fact, we envision that there is an entire novel spectrum of model architectures interpolating between pure FMs and (various instantiations of) GNNs. 
Here we propose one simple implementation of such an architecture which we dub \factorgnns. ~\cref{fig:refactorgnn_architecture} gives an overview of \textsc{ReFactor} GNNs.

\paragraph{The \textsc{ReFactor} Layer}{
A \factorgnn contains $L$ \textsc{ReFactor} layers, that we derive from ~\cref{the:mp_fm}. 
Aligning with the notations in \cref{sec:gnn_kbc}, 
given a KG $\mathcal{T}$ and entity representations $h^{l-1}\in\mathbb{R}^{|\mathcal{E}|\times K}$,
the \textsc{ReFactor} layer computes the representation of a node $v$ as follows:
\begin{equation} 
\label{eq:refactor_layer}
h^l[v] = q^l(\mathcal{T}, h^{l-1})[v] = h^{l-1}[v] - \beta n^l[v] + \alpha \sum_{(r,w) \in \mathcal{N}^1[v]} q_M^l(h^{l-1}[v], r, h^{l-1}[w]),
\end{equation}
where the terms $n^l$ and $q^l_{\mathrm{M}}$ are derived from \cref{eq:fm_normaliser} and \cref{eq:fm_qm}, respectively. 
We note that \factorgnns treat incoming and outgoing neighbourhoods differently instead of treating them equally as in the R-GCN, the first GNN on multi-relational graphs~\citep{schlichtkrull2018modeling}. 

\cref{eq:refactor_layer} describes the full batch setting, which can be expensive if the KG contains many edges.
Therefore, in practice, whenever the graph is big, we adopt a stochastic evaluation of the \textsc{ReFactor} layer by decomposing the evaluation into several mini-batches.
We partition $\mathcal{T}$ into a set of computationally tractable mini-batches.
For each of them, we restrict the neighbourhoods to the subgraph induced by it and readjust the computation of $n^l[v]$ to include only entities and edges present in
it.
We leave the investigation of other stochastic strategies~(e.g. by taking Monte Carlo estimations of the expectations in \cref{eq:fm_normaliser}) to future work.
Finally, we cascade the mini-batch evaluation to produce one full layer evaluation.
}

\paragraph{Training}
The learnable parameters of \factorgnns are the relation embeddings $\psi$, which parameterise $g(r)$ in $q^l_M, l \in [1, L]$. 
Inspired by ~\cite{Fey2021GNNAutoScaleSA,you2020l2}, we learn $\psi$ by layer-wise~(stochastic) gradient descent.
This is in contrast to conventional GNN training, where we need to backpropagate through all the layers.  
A (full-batch) GD training dynamic for $\psi$ can be written  as $\psi_{t+1} = \psi_{t} - \eta \nabla \mathcal{L}_t(\psi_t)$, where $
\mathcal{L}_t(\psi_t) = - |\mathcal{T}|^{-1} \sum_{\mathcal{T}}  \log P_{\psi_t}(w|v,r)$, with:
\begin{equation*}
P_{\psi_t}(w|v, r) = \mathrm{Softmax}(\Gamma(v, r, \cdot))[w],
\qquad 
\Gamma(v, r, w) = \langle h^{t}[v], h^{t}[w], g_{\psi_t}(r) \rangle
\end{equation*}
and the node state update as
\begin{equation}
\begin{split}
    h^{t} = 
    \left\lbrace
    \begin{array}{lr}
        X & \text{if} \; t\bmod L=0 
        \\
        q^{t \bmod L}( \mathcal{T}, h^{t-1})  & \text{otherwise.}
    \end{array} 
    \right.\\ 
\end{split}
\end{equation}
Implementation-wise, such a training dynamic equals to using an external memory for storing historical node states
$h^{t-1}$ akin to the procedure introduced in \cite{Fey2021GNNAutoScaleSA}. 
The cached node states can then be queried to compute $h^t$ using \cref{eq:refactor_layer}.
From this perspective,
we periodically clear \emph{the node state cache} every $L$ full batches to force the model to predict based on on-the-fly $L$-layer message-passing. 
After training, we obtain $\psi^*$ and perform inference by running $L$-layer message-passing with $\psi^*$.
\begin{figure}
\begin{center}
\includegraphics[scale=0.1]{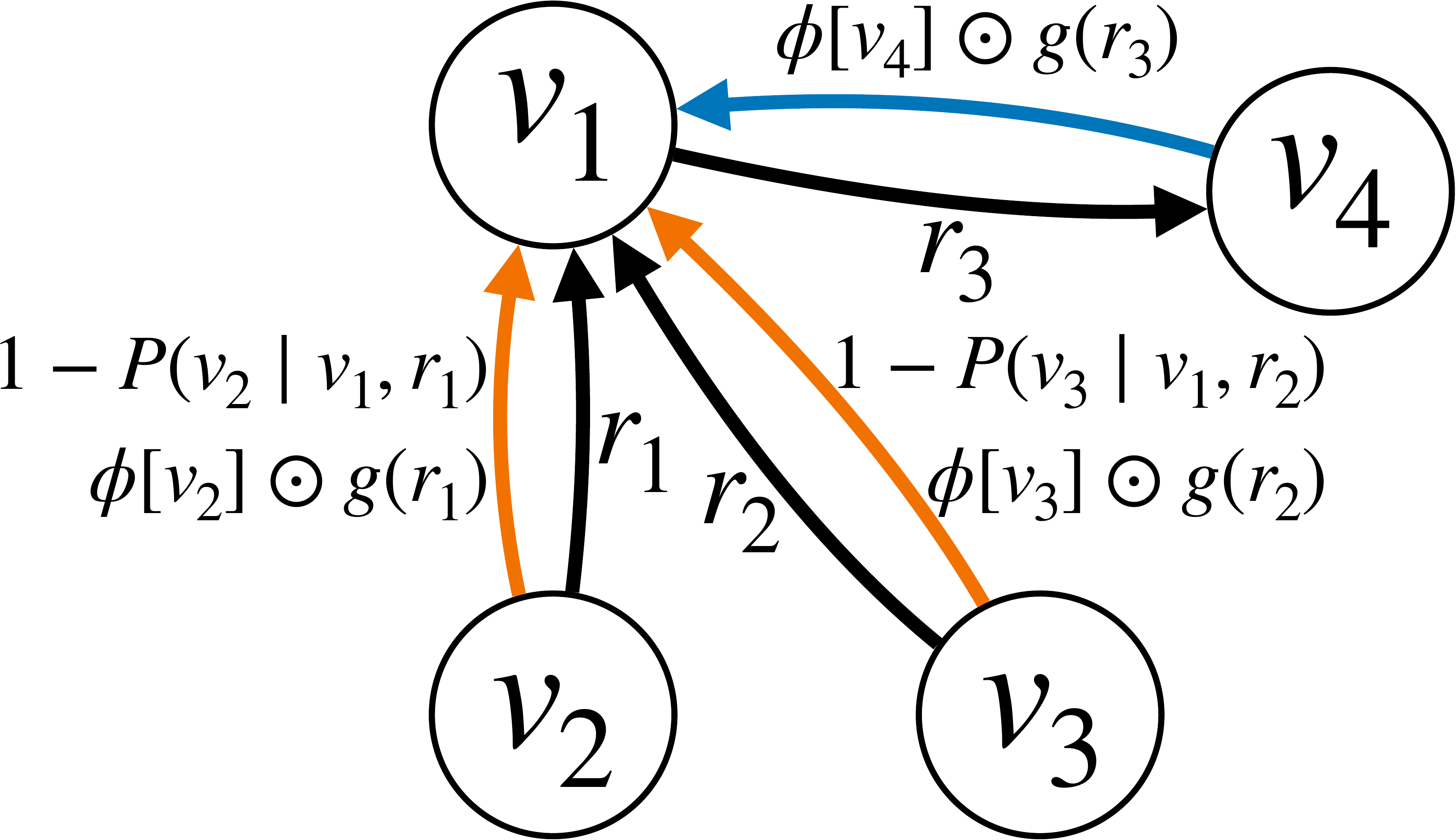}
\hspace{1cm}
\includegraphics[scale=0.1]{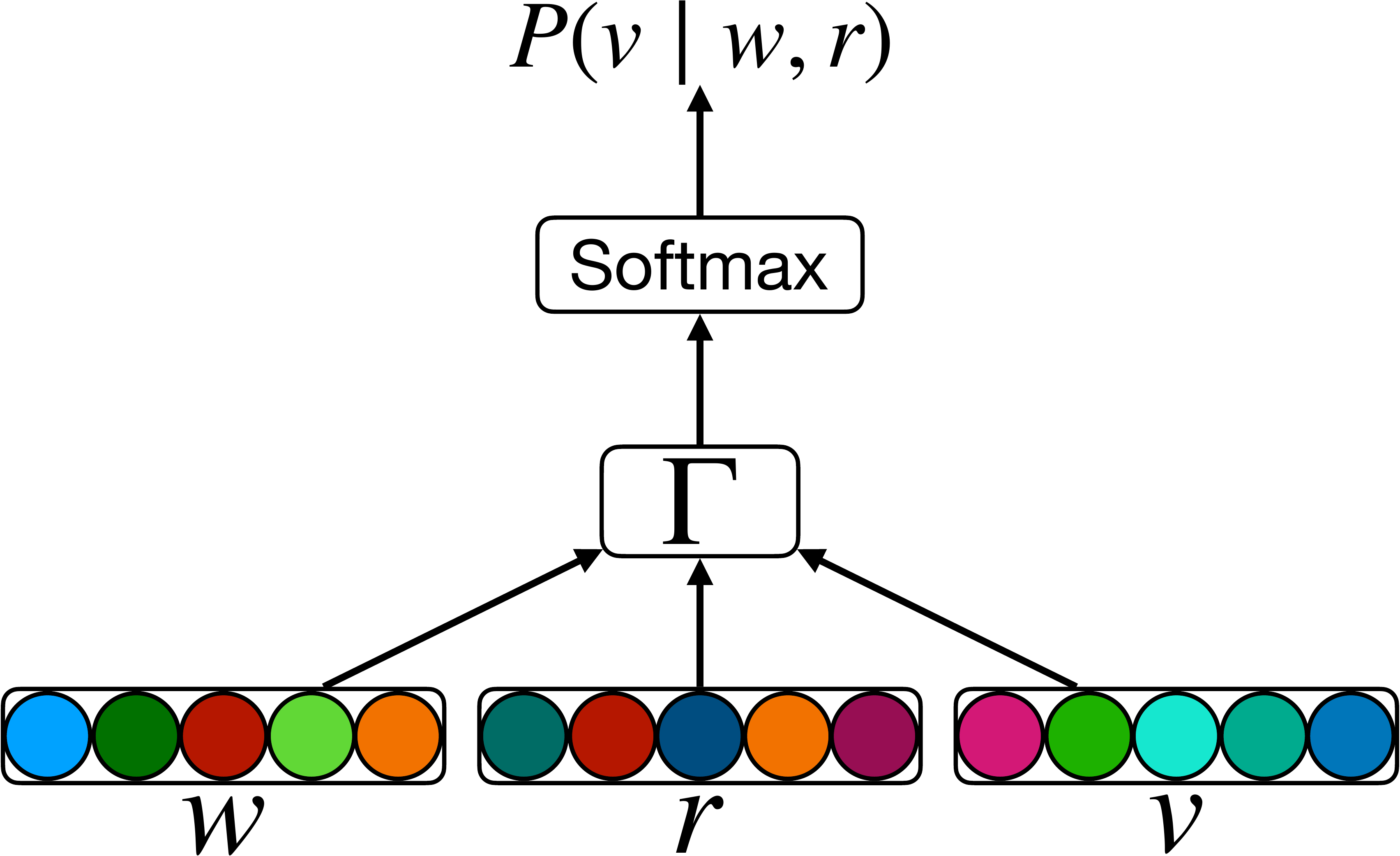}
\end{center}
\caption{ReFactor GNN architecture -- the left figure describes the messages (coloured edges) used to update the representation of node $v_{1}$, which depend on the type of relationship between the sender nodes and $v_{1}$ in the graph $G = \{ (v_2, r_1, v_1), (v_3, r_2, v_1), (v_1, r_3, v_4) \}$; the right figure describes the computation graph for calculating $P(v \mid w, r)$, where $v, w \in \mathcal{E}$ and $r \in \mathcal{R}$: the embedding representations of $w$, $r$, and $v$ are used to score the edge $(w, r, v)$ via the scoring function $\Gamma$, which is then normalised via the $\text{Softmax}$ function.}
\label{fig:refactorgnn_architecture}
\end{figure}
In general, $L$ determines the number of effective message-passing layers in \factorgnns.
A larger $L$ enables \factorgnns to fuse information from more hops of neighbourhoods into the final node representations.
In the meantime, it reduces the inductive applicability of \factorgnns due to over-smoothing and computational requirements.
In the extreme case of $L=\infty$, \textit{where we never clear the node state cache during training}, the final cached node states will be used for inference.
Note that this latter inference regime is inherently transductive since there will be no cached states for new nodes. Future work may explore a more streamlined implementation by simply resetting the entity embeddings periodically.
\section{Experiments}
We perform experiments to answer the following questions regarding \factorgnns:
\begin{itemize}[leftmargin=*]
    \item {\bf Q1.} \factorgnns are derived from a message-passing reformulation of FMs: do they also inherit FMs' predictive accuracy in \emph{transductive} KGC tasks? (\cref{sec:refactorgnn_transductive})
    \item {\bf Q2.} \factorgnns ``inductivise'' FMs. Are they more statistically accurate than other GNN baselines in \emph{inductive} KGC tasks? (\cref{sec:refactorgnn_inductive})
    \item {\bf Q3.} The term $n[v]$ involves nodes that are not in the 1-hop neighbourhood. Is such \emph{augmented message passing}~\citep{augmented-mp} necessary for good KGC performance? (\cref{sec:augmp})
\end{itemize}
For transductive experiments, we used three well-established KGC datasets: \textit{UMLS}~\citep{kemp2006concepts}, \textit{CoDEx-S}~\citep{safavi2020codex}, and \textit{FB15K237}~\citep{toutanova2015observed}. 
For inductive experiments, we used the inductive KGC benchmarks introduced by GraIL~\citep{Teru2020InductiveRP}, which include 12 pairs of knowledge graphs: (\textit{FB15K237\_v$i$}, \textit{FB15K237\_v$i$\_ind}), (\textit{WN18RR\_v$i$}, \textit{WN18RR\_v$i$\_ind}), and (\textit{NELL\_v$i$}, \textit{NELL\_v$i$\_ind}), where $i \in [1, 2, 3, 4]$, and (\textit{\_v$i$}, \textit{\_v$i$\_ind}) represents a pair of graphs with a shared relation vocabulary and non-overlapping entities.
We follow the standard KGC evaluation protocol by fully ranking all the candidate entities and computing two metrics using the ranks of the ground-truth entities: Mean Reciprocal Ranking~(MRR) and Hit Ratios at Top K~(Hits@$K$) with $K \in [1, 3, 10]$.
For the inductive KGC, we additionally consider the partial-ranking evaluation protocol used by GraIL for a fair comparison. 
Empirically, we find full ranking more difficult than partial ranking, and thus more suitable for reflecting the differences among models on GraIL datasets -- we would like to call for future work on GraIL datasets to also adopt a full ranking protocol on these datasets.

Our \textit{transductive} experiments used $L=\infty$, i.e. node states cache is never cleared, as we wanted to see if \factorgnns ($L=\infty$) can reach the performance of the FMs; on the other hand, in our \textit{inductive} experiments, we used \factorgnns with $L \in \{1, 2, 3, 6, 9\}$, since we wanted to test their performances in inductive settings akin to standard GNNs.
We used a hidden size of 768 for the node representations.
All the models are trained using $[128, 512]$ in-batch negative samples and one global negative node for each positive link.
We performed a grid search over the other hyper-parameters and selected the best configuration based on the validation MRR. 
Since training deep GNNs with full-graph message passing might be slow for large knowledge graphs, we follow the literature~\citep{hamilton2017inductive,ladies2019,graphsaint-iclr20} to sample sub-graphs for training GNNs as indicated by \Cref{eq:gnn_sub-graph}.
Considering that sampling on the fly often prevents high utilisation of GPUs, we resort to a two-stage process: we first sampled and serialised sub-graphs around the target edges in the mini-batches; we then trained the GNNs with the serialised sub-graphs. 
To ensure we have sufficient sub-graphs for training the models, we sampled for 20 epochs for each knowledge graph, \ie 20 full passes over the full graph. 
The sub-graph sampler we currently used is LADIES \citep{ladies2019}.

\subsection{\factorgnns for Transductive Learning (Q1)}
\label{sec:refactorgnn_transductive}
\factorgnns are derived from the message-passing reformulation of FMs. 
We expect them to approximate the performance of FMs for transductive KGC tasks.
To verify this, we run experiments on the datasets UMLS, CoDEx-S, and FB15K237.
For a fair comparison, we use \cref{eq:distmult} as the decoder  and consider i) lookup embedding table as the entity encoder, which forms the FM when combined with the decoder (\cref{sec:bg_fm}), and ii) \factorgnns as the entity encoder. 
\factorgnns are trained with $L=\infty$, i.e. we never clear the node state cache.
Since transductive KGC tasks do not involve new entities, the node state cache in \factorgnns can be directly used for link prediction.
\cref{tab:transductive_kbc} summarises the result. 
We observe that \factorgnns achieve a similar or slightly better performance compared to the FM.
This shows that \factorgnns are able to capture the essence of FMs and thus maintain strong at transductive KGC.
\begin{table}
\centering
\begin{tabular}{rccc}
\toprule
{\bf Entity Encoder} & {\bf UMLS} & {\bf CoDEx-S} & {\bf FB15K237} \\
\midrule
R-GCN &  - & 0.33 & 0.25 \\
Lookup (FM, specif. DistMult) & 0.90 & 0.43    & 0.30\\
\factorgnns ($L = \infty$)    & 0.93 & 0.44    & 0.33\\
\bottomrule
\end{tabular}
\caption{Test MRR for transductive KGC tasks. }
\label{tab:transductive_kbc}
\end{table}

\subsection{\factorgnns for Inductive Learning (Q2)}
\label{sec:refactorgnn_inductive}
Despite FMs' good empirical performance on transductive KGC tasks, they fail to be inductive as GNNs.
According to our reformulation, this is due to the infinite message-passing layers hidden in FMs' optimisation.
Discarding infinite message-passing layers, \factorgnns enable FMs to perform inductive reasoning tasks by learning to use a finite set of message-passing layers for prediction similarly to GNNs.

Here we present experiments to verify \factorgnns's capability for inductive reasoning.
Specifically, we study the task of inductive KGC and investigate whether \factorgnns can generalise to unseen entities.
Following \cite{Teru2020InductiveRP}, on GraIL datasets, we trained models on the original graph, and run 0-shot link prediction on the \textit{\_ind} test graph. 
Similar to the transductive experiments, we use \cref{eq:distmult} as the decoder and vary the entity encoder.
We denote three-layer \factorgnns as \texttt{\factorgnns(3)} and six-layer \factorgnns as \texttt{\factorgnns(6)}.
We consider several baseline entity encoders: 
i) \texttt{no-pretrain}, models without any pretraining on the original graph; 
ii) \texttt{GAT(3)}, three-layer graph attention network~\citep{velivckovic2018graph};  
iii) \textit{GAT(6)}, six-layer graph attention network;  
iv) \texttt{GraIL}, a sub-graph-based relational GNN~\cite{Teru2020InductiveRP};  
v) \texttt{NBFNet}, a path-based GNN~\cite{zhaochengzhu2021}, current SoTA on GraIL datasets.
In addition to randomly initialised vectors as the node features, we also used textual node features, RoBERTa~\cite{DBLP:journals/corr/abs-1907-11692} Encodings of the entity descriptions, which are produced by SentenceBERT~\cite{reimers-2019-sentence-bert}.
Due to space reason, we present the results on (\textit{FB15K237\_v$1$}, \textit{FB15K237\_v$1$\_ind}) in ~\cref{fig:inductive_link_prediction}. Results on other datasets are similar and can be found in the appendix. 
We can see that without textual node features, \factorgnns perform better than GraIL (+23\%); with textual node features, \factorgnns outperform both GraIL (+43\%) and NBFNet (+10\%), achieving new SoTA results on inductive KGC.

\begin{figure}
    \centering
    \includegraphics[scale=0.28]{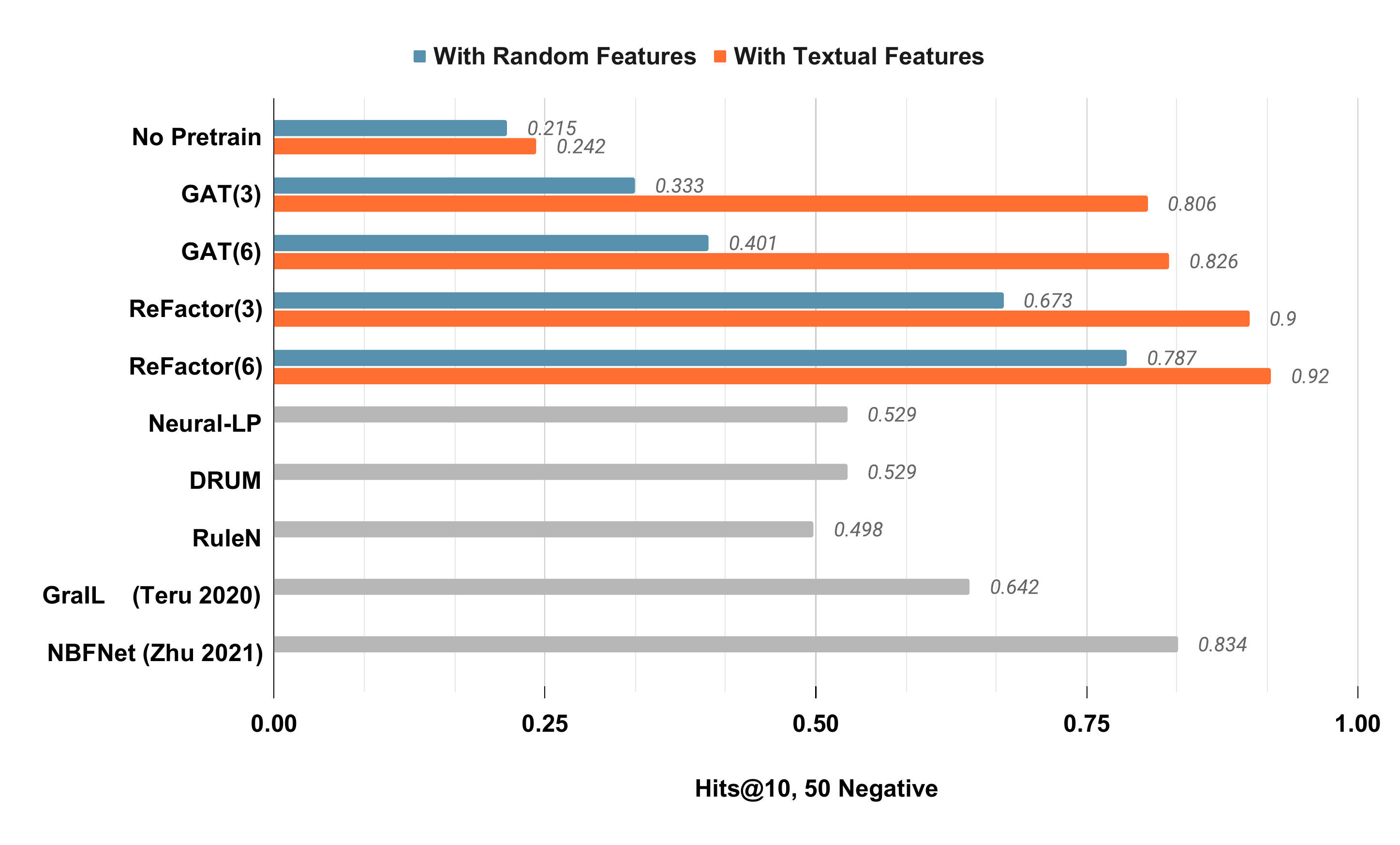}
    \caption{Inductive KGC Performance. Models are trained on the KG \textit{FB15K237\_v1} and tested on another KG \textit{FB15K237\_v1\_ind}, where the entities are completely new. The results of GraIL and NBFNet are taken from ~\cite{zhaochengzhu2021}. The grey bars indicate methods that are not devised to incorporate node features.}
    \label{fig:inductive_link_prediction}
\end{figure}

\paragraph{Performance vs Parameter Efficiency as \#Message-Passing Layers Increases}{Usually, as the number of message-passing layers increases in GNNs, the over-smoothing issue occurs while the computational cost also increases exponentially.
\factorgnns avoid this by layer-wise training and sharing the weights across layers. 
\begin{figure}
    \centering
    \includegraphics[scale=0.27]{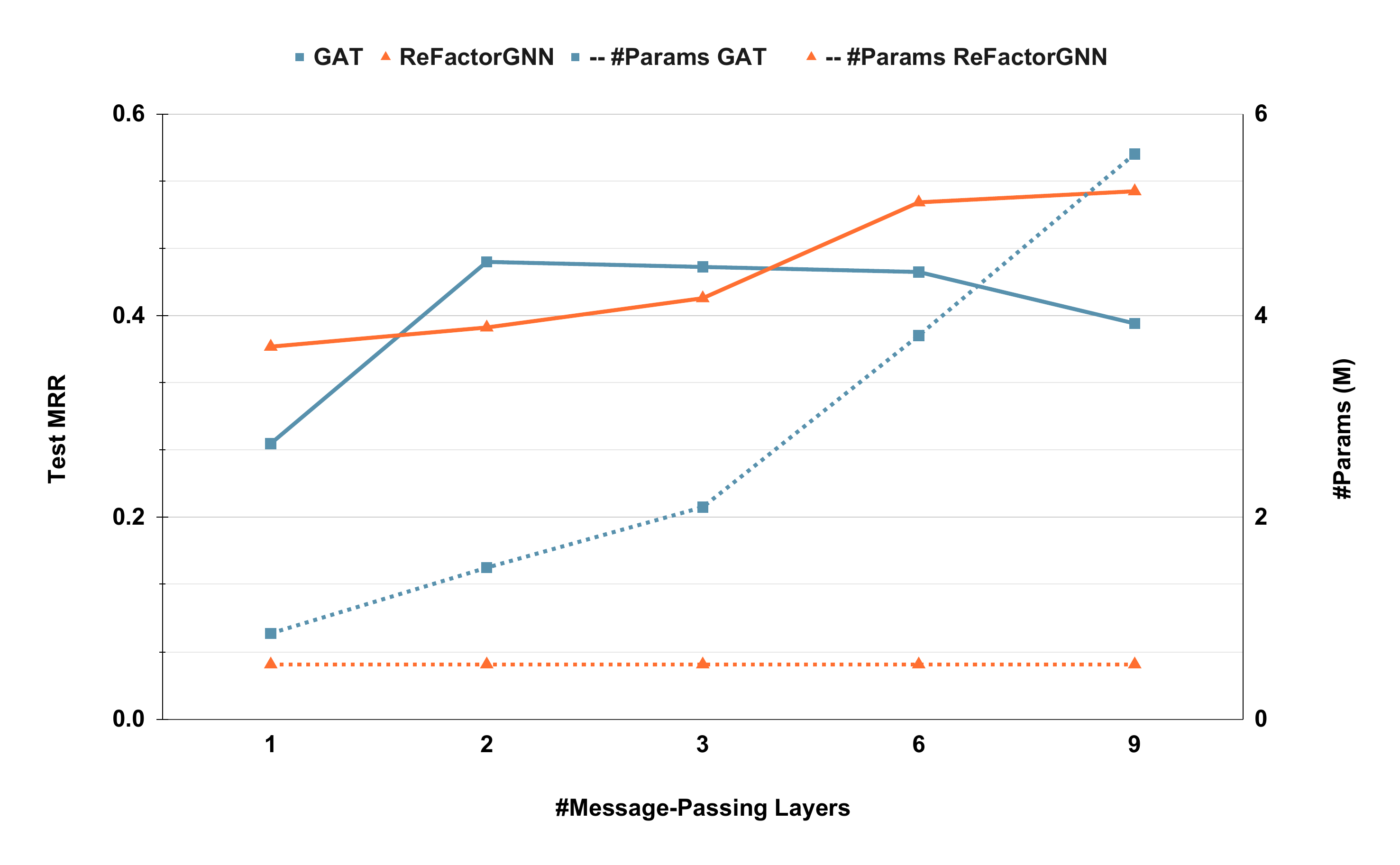}
    \caption{Performance vs Parameter Efficiency as \#Layers Increases on  \textit{FB15K237\_v1}. The left axis is Test MRR while the right axis is \#Parameters. The solid lines and dashed lines indicate the changes of Test MRR and the changes of \#Parameters. }     \label{fig:parameter_efficiency}
\end{figure}
Here we compare \factorgnns with $\{1, 3, 6, 9\}$ message-passing layer(s) with same-depth GATs -- results are summarised in \cref{fig:parameter_efficiency}.
We observe that increasing the number of message-passing layers in GATs does not necessarily improve the predictive accuracy -- the best results were obtained with 3 message-passing layers on \textit{FB15K237\_v1} while using 6 and 9 layers leads to performance degradation. 
On the other hand, \factorgnns obtain consistent improvements when increasing \#Layers from 1 to 3, 6, and 9. \factorgnns$(6, 6)$ and $(9, 9)$ clearly outperform their GAT counterparts.
Most importantly, \factorgnns are more parameter-efficient than GATs, with a constant \#Parameters as \#Layers increases.
}
\subsection{Beyond Message-Passing (Q3)}
\label{sec:augmp}
As shown by \cref{the:mp_fm}, \factorgnns contain not only terms capturing information flow from the 1-hop neighbourhood, which falls into the classic message-passing framework, but also a term $n[v]$ that involve nodes outside the 1-hop neighbourhood. 
The term $n[v]$ can be treated as \emph{augmented message-passing} on a dynamically rewired graph~\cite{augmented-mp}.
Here we perform ablation experiments to measure the impact of the $n[v]$ term.
\cref{tab:ablation_augmp} summarises the ablation results: we can see that, without the term $n[v]$, \factorgnns with random vectors as node features yield a 2\% lower MRR, while \factorgnns with RoBERTa textual encodings as node features produce a 7\% lower MRR.
This suggests that augmented message-passing also plays a significant role in \factorgnns' generalisation properties in downstream link prediction tasks.
Future work might gain more insights by further dissecting the $n[v]$ term.
\begin{table}[tbp]
\centering
\begin{tabular}{rcc}
\toprule
{\bf Test MRR}             & {\bf With Random Features} & {\bf With Textual Features} \\
\midrule
with $n[v]$&  0.425 & 0.486    \\ 
without $n[v]$ & 0.418 & 0.452 \\
\bottomrule
\end{tabular}
\caption{Ablation on $n[v]$ for \factorgnns(6) trained on \textit{FB15K237\_v1}.}
\label{tab:ablation_augmp}
\end{table}

\section{Related Work}
\paragraph{Multi-Relational Graph Representation Learning} 
Multi-relational graph representation learning concerns graphs with various edge types. 
Another relevant line of work would be representation learning over heterogeneous graphs, where node types are also considered.
Previous work on multi-relational graph representation learning focused either on FMs~\citep{DBLP:conf/icml/NickelTK11,pmlr-v48-trouillon16,yang2014embedding,lacroix2018canonical,DBLP:conf/aaai/NickelRP16,DBLP:conf/aaai/DettmersMS018,DBLP:conf/naacl/NguyenNNP18,chen2021relation} or on GNN-based models~\citep{schlichtkrull2018modeling,DBLP:conf/iclr/XuFJXSD20,DBLP:conf/aaai/ZhangZZ0XH20,DBLP:journals/corr/abs-2109-11800}.
Similar to a recent finding in a benchmark study over heterogeneous GNNs~\cite{lv2021hgb}, where the best choices of GNNs for heterogeneous graphs seem to regress to simple homogeneous GNN baselines, the progress of multi-relational graph representation learning also mingles with FMs, the very classic multi-relational link predictors.
Recently, FMs were found to be significantly more accurate than GNNs in KGC tasks, when coupled with specific training strategies~\citep{Ruffinelli2020You,jain2020again,lacroix2018canonical}.
While more advanced GNNs~\cite{zhaochengzhu2021} for KBC are showing promise at the cost of extra algorithm complexity, little effort has been devoted to establishing the links between plain GNNs and FMs, which are strong multi-relational link predictors despite their simplicity.
Our work aims to \emph{align} GNNs with FMs so that we can combine the strengths from both families of models.

\paragraph{Relationships between FMs and GNNs}
A very recent work~\cite{Srinivasan2020On} builds a theoretical link between structural GNNs and node (positional) embeddings.
However, on one end of the link, the second model category encompasses not merely factorisation-based models but also many practical graph neural networks, between which the connection is unknown.
Our work instead offers a more practical link between positional node embeddings produced by FMs and positional node embeddings produced by GNNs, while at the same time focusing on KGC.
Beyond FMs in KGC, using graph signal processing theory, \cite{shen2021powerful} show that matrix factorisation (MF) based recommender models correspond to ideal low-pass graph convolutional filters.
They also find infinite neighbourhood coverage in MF although using a completely different approach and focusing on a different domain in contrast to our work.
\paragraph{Message Passing}
Message-passing is itself a broad terminology -- people generally talk about it under two different contexts. 
Firstly, as a computational technique, message passing allows recursively decomposing a global function into simple local, parallelisable computations~\citep{mackay2003information}, thus being widely used for solving inference problems in a graphical model.
Specifically, we note that message passing-based inference techniques were proposed for matrix completion-based recommendation~\citep{kim2010imp} and Bayesian Boolean data decomposition~\citep{ravanbakhsh2016boolean} in the pre-deep-learning era. 
Secondly, as a paradigm of parameterising learnable functions over \emph{graph-structured data}, message-passing has recently been used to provide a unified reformulation~\citep{gilmer2017neural} for various GNN architectures, including Graph Attention Networks~\citep{velivckovic2018graph}, Gated Graph Neural Networks~\citep{DBLP:journals/corr/LiTBZ15}, and Graph Convolutional Networks~\citep{kipf2016semi}.
In this work, we show that FMs can also be cast as a special type of message-passing GNNs by considering the gradient descent updates~\citep{Bottou2012} over node embeddings as message-passing operations between nodes.
To the best of our knowledge, our work is the first to provide such connections between FMs and message-passing GNNs.
We show that FMs can be seen as instances of GNNs, with a characteristic feature about the nodes being considered during the message-passing process: our \factorgnns can be seen as using an \emph{Augmented Message-Passing} process on a dynamically re-wired graph~\citep{augmented-mp}.

\section{Conclusion \& Future Work}
\paragraph{Conclusion.}
Motivated by the need of understanding FMs and GNNs despite the sharp differences in their analytical forms, our work establishes a link between FMs and GNNs on the task of multi-relational link prediction.
The reformulation of FMs as GNNs addresses the question why FMs are stronger multi-relational link predictors compared to plain GNNs.
Guided by the reformulation, we further propose a new variant of GNNs, \factorgnns, which combines the strengths of both GNNs and classic FMs.
Empirical experiments show that \factorgnns produce significantly more accurate results than our GNN baselines on inductive link prediction tasks.

\paragraph{Limitations.} 
Since we adopted a two-stage (sub-graph serialisation  and then model training) approach instead of online sampling, there can be side effects from the low sub-graph diversity. In our experiments, we used LADIES~\citep{ladies2019} for sub-graph sampling.
Experiments with different sub-graph sampling algorithms, such as GraphSaint~\citep{graphsaint-iclr20} might affect the downstream link prediction results.
Furthermore, it would be interesting to analyse decoders other than DistMult, as well as additional optimisation schemes beyond SGD and AdaGrad.
We do not dive deeper into the expressiveness of \factorgnns. Nevertheless, we offer a brief discussion in \Cref{sec:expressiveness}.

\paragraph{Future Work.}
The most direct future work would be using the insight to develop more sophisticated models at the crossroad between FMs and GNNs, e.g. by further parameterising the message/update function.
One implication from our work is that reformulating FMs as message-passing enables the idea of ``learning to factorize''. 
This might broaden the usages of FMs, going beyond link prediction, to tasks such as graph classification. 
Another implication comes from our approach of unpacking embedding updates into a series of message-passing operations. 
This approach can be generalised to other dot-product-based models that use embedding layers for processing the inputs, e.g. transformers.
In this sense, our work paves the way towards understanding complicated attention-based models e.g. transformers. 
Although transformers can be treated as GNNs over fully-connected graphs, where a sentence would be a graph and its tokens would be the nodes, the message-passing is limited to within each sentence under this view.
We instead envision cross-sentence message-passing by reformulating the updates of the token embedding layer  in transformers.
In general, the direction of organising FMs, GNNs, and transformers under the same framework will allow a better understanding and potentially unifying all of the three models.

\section*{Acknowledgements}
We would like to thank all reviewers for their constructive comments. We thank people with whom we had some really interesting chats throughout the project: Zhaocheng Zhu, Floris Geerts, Muhan Zhang, Patrick Lewis, Hanqing Zeng, Shalini Maiti, Yifei Shen, Yinglong Xia, Matthias Fey, Jing Zhu, Marc Deisenroth and Maximilian Nickel. Yihong would like to thank the UCL-Facebook AI Research joint PhD program for generously funding her PhD.
Pasquale and Pontus were partially funded by the European Union’s Horizon 2020 research and innovation programme under grant agreement no. 875160, and by an industry grant from Cisco.

\bibliography{ref}

\begin{thebibliography}{10}

\bibitem{battaglia2018relational}
Peter~W Battaglia, Jessica~B Hamrick, Victor Bapst, Alvaro Sanchez-Gonzalez,
  Vinicius Zambaldi, Mateusz Malinowski, Andrea Tacchetti, David Raposo, Adam
  Santoro, Ryan Faulkner, et~al.
\newblock Relational inductive biases, deep learning, and graph networks.
\newblock {\em arXiv preprint arXiv:1806.01261}, 2018.

\bibitem{Bottou2012}
L{\'e}on Bottou.
\newblock {\em Stochastic Gradient Descent Tricks}, pages 421--436.
\newblock Springer Berlin Heidelberg, Berlin, Heidelberg, 2012.

\bibitem{NEURIPS2023_6450ea28}
Yihong Chen, Kelly Marchisio, Roberta Raileanu, David Adelani, Pontus Lars~Erik
  Saito~Stenetorp, Sebastian Riedel, and Mikel Artetxe.
\newblock Improving language plasticity via pretraining with active forgetting.
\newblock In A.~Oh, T.~Naumann, A.~Globerson, K.~Saenko, M.~Hardt, and
  S.~Levine, editors, {\em Advances in Neural Information Processing Systems},
  volume~36, pages 31543--31557. Curran Associates, Inc., 2023.

\bibitem{chen2021relation}
Yihong Chen, Pasquale Minervini, Sebastian Riedel, and Pontus Stenetorp.
\newblock Relation prediction as an auxiliary training objective for improving
  multi-relational graph representations.
\newblock In {\em 3rd Conference on Automated Knowledge Base Construction},
  2021.

\bibitem{DBLP:conf/aaai/DettmersMS018}
Tim Dettmers, Pasquale Minervini, Pontus Stenetorp, and Sebastian Riedel.
\newblock Convolutional 2d knowledge graph embeddings.
\newblock In {\em {AAAI}}, pages 1811--1818. {AAAI} Press, 2018.

\bibitem{JMLR:v12:duchi11a}
John Duchi, Elad Hazan, and Yoram Singer.
\newblock Adaptive subgradient methods for online learning and stochastic
  optimization.
\newblock {\em Journal of Machine Learning Research}, 12(61):2121--2159, 2011.

\bibitem{Fey2021GNNAutoScaleSA}
Matthias Fey, Jan~Eric Lenssen, Frank Weichert, and Jure Leskovec.
\newblock Gnnautoscale: Scalable and expressive graph neural networks via
  historical embeddings.
\newblock In {\em ICML}, 2021.

\bibitem{geerts2021expressiveness}
Floris Geerts and Juan~L Reutter.
\newblock Expressiveness and approximation properties of graph neural networks.
\newblock In {\em International Conference on Learning Representations}, 2021.

\bibitem{gilmer2017neural}
Justin Gilmer, Samuel~S Schoenholz, Patrick~F Riley, Oriol Vinyals, and
  George~E Dahl.
\newblock Neural message passing for quantum chemistry.
\newblock In {\em International conference on machine learning}, pages
  1263--1272. PMLR, 2017.

\bibitem{Gori2005ANM}
Marco Gori, Gabriele Monfardini, and Franco Scarselli.
\newblock A new model for learning in graph domains.
\newblock {\em Proceedings. 2005 IEEE International Joint Conference on Neural
  Networks, 2005.}, 2:729--734 vol. 2, 2005.

\bibitem{grlbook}
William~L. Hamilton.
\newblock Graph representation learning.
\newblock {\em Synthesis Lectures on Artificial Intelligence and Machine
  Learning}, 14(3):1--159.

\bibitem{hamilton2017inductive}
William~L Hamilton, Rex Ying, and Jure Leskovec.
\newblock Inductive representation learning on large graphs.
\newblock In {\em Proceedings of the 31st International Conference on Neural
  Information Processing Systems}, pages 1025--1035, 2017.

\bibitem{gpt_gnn}
Ziniu Hu, Yuxiao Dong, Kuansan Wang, Kai-Wei Chang, and Yizhou Sun.
\newblock Gpt-gnn: Generative pre-training of graph neural networks.
\newblock In {\em Proceedings of the 26th ACM SIGKDD Conference on Knowledge
  Discovery and Data Mining}, 2020.

\bibitem{jain2020again}
Prachi Jain, Sushant Rathi, Mausam, and Soumen Chakrabarti.
\newblock Knowledge base completion: Baseline strikes back (again).
\newblock {\em CoRR}, abs/2005.00804, 2020.

\bibitem{kazemi2018simple}
Seyed~Mehran Kazemi and David Poole.
\newblock Simple embedding for link prediction in knowledge graphs.
\newblock {\em Advances in neural information processing systems}, 31, 2018.

\bibitem{kemp2006concepts}
Charles Kemp, Joshua~B. Tenenbaum, Thomas~L. Griffiths, Takeshi Yamada, and
  Naonori Ueda.
\newblock Learning systems of concepts with an infinite relational model.
\newblock In {\em {AAAI}}, pages 381--388. {AAAI} Press, 2006.

\bibitem{kim2010imp}
Byung-Hak Kim, Arvind Yedla, and Henry~D Pfister.
\newblock Imp: A message-passing algorithm for matrix completion.
\newblock In {\em 2010 6th International Symposium on Turbo Codes \& Iterative
  Information Processing}, pages 462--466. IEEE, 2010.

\bibitem{kipf2016semi}
Thomas~N Kipf and Max Welling.
\newblock Semi-supervised classification with graph convolutional networks.
\newblock {\em arXiv preprint arXiv:1609.02907}, 2016.

\bibitem{lacroix2018canonical}
Timoth{\'{e}}e Lacroix, Nicolas Usunier, and Guillaume Obozinski.
\newblock Canonical tensor decomposition for knowledge base completion.
\newblock In {\em {ICML}}, volume~80 of {\em Proceedings of Machine Learning
  Research}, pages 2869--2878. {PMLR}, 2018.

\bibitem{DBLP:journals/corr/abs-2109-11800}
Ren Li, Yanan Cao, Qiannan Zhu, Guanqun Bi, Fang Fang, Yi~Liu, and Qian Li.
\newblock How does knowledge graph embedding extrapolate to unseen data: a
  semantic evidence view.
\newblock {\em CoRR}, abs/2109.11800, 2021.

\bibitem{DBLP:journals/corr/LiTBZ15}
Yujia Li, Daniel Tarlow, Marc Brockschmidt, and Richard~S. Zemel.
\newblock Gated graph sequence neural networks.
\newblock In {\em {ICLR} (Poster)}, 2016.

\bibitem{DBLP:journals/corr/abs-1907-11692}
Yinhan Liu, Myle Ott, Naman Goyal, Jingfei Du, Mandar Joshi, Danqi Chen, Omer
  Levy, Mike Lewis, Luke Zettlemoyer, and Veselin Stoyanov.
\newblock Roberta: {A} robustly optimized {BERT} pretraining approach.
\newblock {\em CoRR}, abs/1907.11692, 2019.

\bibitem{lv2021hgb}
Qingsong Lv, Ming Ding, Qiang Liu, Yuxiang Chen, Wenzheng Feng, Siming He,
  Chang Zhou, Jianguo Jiang, Yuxiao Dong, and Jie Tang.
\newblock Are we really making much progress? revisiting, benchmarking and
  refining heterogeneous graph neural networks.
\newblock In {\em Proceedings of the 27th ACM SIGKDD Conference on Knowledge
  Discovery and Data Mining}, KDD '21, page 1150–1160, New York, NY, USA,
  2021. Association for Computing Machinery.

\bibitem{mackay2003information}
David~JC MacKay.
\newblock {\em Information theory, inference and learning algorithms}.
\newblock Cambridge university press, 2003.

\bibitem{nathani2019learning}
Deepak Nathani, Jatin Chauhan, Charu Sharma, and Manohar Kaul.
\newblock Learning attention-based embeddings for relation prediction in
  knowledge graphs.
\newblock In {\em Proceedings of the 57th Annual Meeting of the Association for
  Computational Linguistics}, pages 4710--4723, 2019.

\bibitem{DBLP:conf/naacl/NguyenNNP18}
Dai~Quoc Nguyen, Tu~Dinh Nguyen, Dat~Quoc Nguyen, and Dinh~Q. Phung.
\newblock A novel embedding model for knowledge base completion based on
  convolutional neural network.
\newblock In {\em {NAACL-HLT} {(2)}}, pages 327--333. Association for
  Computational Linguistics, 2018.

\bibitem{DBLP:journals/pieee/Nickel0TG16}
Maximilian Nickel, Kevin Murphy, Volker Tresp, and Evgeniy Gabrilovich.
\newblock A review of relational machine learning for knowledge graphs.
\newblock {\em Proc. {IEEE}}, 104(1):11--33, 2016.

\bibitem{DBLP:conf/aaai/NickelRP16}
Maximilian Nickel, Lorenzo Rosasco, and Tomaso~A. Poggio.
\newblock Holographic embeddings of knowledge graphs.
\newblock In {\em {AAAI}}, pages 1955--1961. {AAAI} Press, 2016.

\bibitem{DBLP:conf/icml/NickelTK11}
Maximilian Nickel, Volker Tresp, and Hans{-}Peter Kriegel.
\newblock A three-way model for collective learning on multi-relational data.
\newblock In {\em {ICML}}, pages 809--816. Omnipress, 2011.

\bibitem{ravanbakhsh2016boolean}
Siamak Ravanbakhsh, Barnab{\'a}s P{\'o}czos, and Russell Greiner.
\newblock Boolean matrix factorization and noisy completion via message
  passing.
\newblock In {\em International Conference on Machine Learning}, pages
  945--954. PMLR, 2016.

\bibitem{reimers-2019-sentence-bert}
Nils Reimers and Iryna Gurevych.
\newblock Sentence-bert: Sentence embeddings using siamese bert-networks.
\newblock In {\em Proceedings of the 2019 Conference on Empirical Methods in
  Natural Language Processing}. Association for Computational Linguistics, 11
  2019.

\bibitem{Ruffinelli2020You}
Daniel Ruffinelli, Samuel Broscheit, and Rainer Gemulla.
\newblock You can teach an old dog new tricks! on training knowledge graph
  embeddings.
\newblock In {\em International Conference on Learning Representations}, 2020.

\bibitem{safavi2020codex}
Tara Safavi and Danai Koutra.
\newblock Codex: A comprehensive knowledge graph completion benchmark.
\newblock In {\em Proceedings of the 2020 Conference on Empirical Methods in
  Natural Language Processing (EMNLP)}, pages 8328--8350, 2020.

\bibitem{DBLP:journals/tnn/ScarselliGTHM09}
Franco Scarselli, Marco Gori, Ah~Chung Tsoi, Markus Hagenbuchner, and Gabriele
  Monfardini.
\newblock The graph neural network model.
\newblock {\em {IEEE} Trans. Neural Networks}, 20(1):61--80, 2009.

\bibitem{schlichtkrull2018modeling}
Michael Schlichtkrull, Thomas~N Kipf, Peter Bloem, Rianne Van Den~Berg, Ivan
  Titov, and Max Welling.
\newblock Modeling relational data with graph convolutional networks.
\newblock In {\em European semantic web conference}, pages 593--607. Springer,
  2018.

\bibitem{shen2021powerful}
Yifei Shen, Yongji Wu, Yao Zhang, Caihua Shan, Jun Zhang, Khaled~B Letaief, and
  Dongsheng Li.
\newblock How powerful is graph convolution for recommendation?
\newblock {\em arXiv preprint arXiv:2108.07567}, 2021.

\bibitem{Srinivasan2020On}
Balasubramaniam Srinivasan and Bruno Ribeiro.
\newblock On the equivalence between positional node embeddings and structural
  graph representations.
\newblock In {\em International Conference on Learning Representations}, 2020.

\bibitem{sun-etal-2020-evaluation}
Zhiqing Sun, Shikhar Vashishth, Soumya Sanyal, Partha Talukdar, and Yiming
  Yang.
\newblock A re-evaluation of knowledge graph completion methods.
\newblock In {\em Proceedings of the 58th Annual Meeting of the Association for
  Computational Linguistics}, pages 5516--5522, Online, July 2020. Association
  for Computational Linguistics.

\bibitem{Teru2020InductiveRP}
Komal~K. Teru, Etienne~G. Denis, and William~L. Hamilton.
\newblock Inductive relation prediction by subgraph reasoning.
\newblock In {\em {ICML}}, volume 119 of {\em Proceedings of Machine Learning
  Research}, pages 9448--9457. {PMLR}, 2020.

\bibitem{toutanova2015observed}
Kristina Toutanova and Danqi Chen.
\newblock Observed versus latent features for knowledge base and text
  inference.
\newblock In {\em Proceedings of the 3rd workshop on continuous vector space
  models and their compositionality}, pages 57--66, 2015.

\bibitem{pmlr-v48-trouillon16}
Théo Trouillon, Johannes Welbl, Sebastian Riedel, Eric Gaussier, and Guillaume
  Bouchard.
\newblock Complex embeddings for simple link prediction.
\newblock In Maria~Florina Balcan and Kilian~Q. Weinberger, editors, {\em
  Proceedings of The 33rd International Conference on Machine Learning},
  volume~48 of {\em Proceedings of Machine Learning Research}, pages
  2071--2080, New York, New York, USA, 20--22 Jun 2016. PMLR.

\bibitem{velivckovic2018graph}
Petar Veli{\v{c}}kovi{\'c}, Guillem Cucurull, Arantxa Casanova, Adriana Romero,
  Pietro Li{\`o}, and Yoshua Bengio.
\newblock Graph attention networks.
\newblock In {\em International Conference on Learning Representations}, 2018.

\bibitem{augmented-mp}
Petar Veličković.
\newblock Message passing all the way up, 2022.

\bibitem{DBLP:conf/iclr/XuHLJ19}
Keyulu Xu, Weihua Hu, Jure Leskovec, and Stefanie Jegelka.
\newblock How powerful are graph neural networks?
\newblock In {\em {ICLR}}. OpenReview.net, 2019.

\bibitem{Xu2020Dynamically}
Xiaoran Xu, Wei Feng, Yunsheng Jiang, Xiaohui Xie, Zhiqing Sun, and Zhi-Hong
  Deng.
\newblock Dynamically pruned message passing networks for large-scale knowledge
  graph reasoning.
\newblock In {\em International Conference on Learning Representations}, 2020.

\bibitem{DBLP:conf/iclr/XuFJXSD20}
Xiaoran Xu, Wei Feng, Yunsheng Jiang, Xiaohui Xie, Zhiqing Sun, and Zhi{-}Hong
  Deng.
\newblock Dynamically pruned message passing networks for large-scale knowledge
  graph reasoning.
\newblock In {\em {ICLR}}. OpenReview.net, 2020.

\bibitem{yang2014embedding}
Bishan Yang, Wen{-}tau Yih, Xiaodong He, Jianfeng Gao, and Li~Deng.
\newblock Embedding entities and relations for learning and inference in
  knowledge bases.
\newblock In {\em {ICLR} (Poster)}, 2015.

\bibitem{you2020l2}
Yuning You, Tianlong Chen, Zhangyang Wang, and Yang Shen.
\newblock L2-gcn: Layer-wise and learned efficient training of graph
  convolutional networks.
\newblock In {\em Proceedings of the IEEE/CVF Conference on Computer Vision and
  Pattern Recognition}, pages 2127--2135, 2020.

\bibitem{graphsaint-iclr20}
Hanqing Zeng, Hongkuan Zhou, Ajitesh Srivastava, Rajgopal Kannan, and Viktor
  Prasanna.
\newblock {GraphSAINT}: Graph sampling based inductive learning method.
\newblock In {\em International Conference on Learning Representations}, 2020.

\bibitem{DBLP:conf/aaai/ZhangZZ0XH20}
Zhao Zhang, Fuzhen Zhuang, Hengshu Zhu, Zhi{-}Ping Shi, Hui Xiong, and Qing He.
\newblock Relational graph neural network with hierarchical attention for
  knowledge graph completion.
\newblock In {\em {AAAI}}, pages 9612--9619. {AAAI} Press, 2020.

\bibitem{zhaochengzhu2021}
Zhaocheng Zhu, Zuobai Zhang, Louis{-}Pascal A.~C. Xhonneux, and Jian Tang.
\newblock Neural bellman-ford networks: {A} general graph neural network
  framework for link prediction.
\newblock {\em CoRR}, abs/2106.06935, 2021.

\bibitem{ladies2019}
Difan Zou, Ziniu Hu, Yewen Wang, Song Jiang, Yizhou Sun, and Quanquan Gu.
\newblock Few-shot representation learning for out-of-vocabulary words.
\newblock In {\em Advances in Neural Information Processing Systems 32: Annual
  Conference on Neural Information Processing Systems, NeurIPS}, 2019.

\end{thebibliography}
\bibliographystyle{plain}

\newpage

\appendix

\section{Theorem 1 Proof}
In this section, we prove Theorem 1, which we restate here for convenience.
\begin{theorem}[Message passing in FMs] 
\label{the:mp_fm_proof}
The gradient descent operator $\GD$ (\ref{eq:fm_node_view}) on the node embeddings of a DistMult model (\cref{eq:fm_entity_encoder}) with the maximum likelihood objective in \cref{eq:objective} and a multi-relational graph $\mathcal{T}$ defined over entities $\mathcal{E}$ induces a message-passing operator whose composing functions are:  
\begin{align} 
 &q_{\mathrm{M}}(\phi[v], r, \phi[w]) = 
 \left\lbrace
\begin{array}{lr}
\phi[w] \odot  g(r) & \text{if} \; (r,w) \in \mathcal{N}_{+}^1[v], \\
(1 - P_\theta (v|w, r)) 
\phi[w] \odot g(r)  & \text{if} \; (r, w) \in \mathcal{N}_-^1[v];
\end{array}
\right.
\\
&q_{\mathrm{A}}(\{m[v, r, w]\, :\, (r,w) \in \mathcal{N}^1[v]\}) = \sum_{(r,w) \in \mathcal{N}^1[v]} m[v,r,w]; \\
&q_{\mathrm{U}}(\phi[v], z[v]) = \phi[v] + \alpha z[v] - \beta 
n[v],
\end{align}
where, defining the sets of triplets $\mathcal{T}^{-v} =\{(s, r, o) \in \mathcal{T} \; : \; s\neq v \wedge o\neq v\}$, 
\begin{equation}
n[v]= 
\frac{|\mathcal{N}_{+}^{1}[v]|}{|\mathcal{T}|}
\mathbb{E}_{
P_{\mathcal{N}_+^{1}[v]}} \mathbb{E}_{u \sim P_{\theta}(\cdot|v, r)} g(r) \odot \phi[u] 
+
\frac{|\mathcal{T}^{-v}|}{|\mathcal{T}|}
\mathbb{E}_{ P_{\mathcal{T}^{-v}}}P_\theta(v|s, r) g(r) \odot \phi[s] ,
\end{equation}
where $P_{\mathcal{N}^{1}_+[v]}$ and $P_{\mathcal{T}^{-v}}$ are the empirical probability distributions associated to the respective sets.
\end{theorem}

\begin{proof}
Remember that we assume that there are no triplets where the source and the target node are the same~(\ie $(v, r, v)$, with $v \in \mathcal{E}$ and $r \in \mathcal{R}$), and let $v \in \mathcal{E}$ be a node in $\mathcal{E}$.
First, let us consider the gradient descent operator GD over $v$'s node embedding $\phi[v]$:
\begin{equation*}
\GD(\phi, \mathcal{T})[v] = \phi[v] + \alpha\sum_{(\GV, \GR, \GW) \in \mathcal{T}}\frac{\partial\log P(\GW|\GV, \GR)}{\partial\phi[v]}.
\end{equation*}
The gradient is a sum over components associated with the triplets $(\GV,\GR, \GW)\in\mathcal{T}$;
based on whether the corresponding triplet involves $v$ in the subject or object position, or does not involve $v$ at all, these components can be grouped into three categories:
\begin{enumerate}
\item Components corresponding to the triplets where $\GV = v \wedge \GW \neq v$. The sum of these components is given by: 
    \begin{align*}
    \sum_{(v, \GR, \GW) \in \mathcal{T}} \frac{\partial \log P(\GW | v, \GR)}{\partial \phi[v]} &= \sum_{(v, \GR, \GW) \in \mathcal{T}} \left[\frac{\partial \Gamma(v, \GR, \GW)}{\partial \phi[v]} - \sum_u P(u|v, \GR) \frac{\partial\Gamma(v, \GR, u)}{\partial \phi[v]}\right] \\
    &= 
    \textcolor{red}{\sum_{(\GR, \GW) \in \mathcal{N}_{+}^1[v] } \phi[\GW] \odot  g(\GR)} \textcolor{blue}{-  \sum_{(v, \GR, \GW) \in \mathcal{T}} \sum_u P(u|v, \GR) g(\GR) \odot \phi[u]}. 
    \end{align*}
\item Components corresponding to the triplets where $\GV \neq v \wedge \GW = v$. The sum of these components is given by:
    \begin{align*}
    \sum_{(\GV, \GR, v) \in \mathcal{T}} \frac{\partial \log P(v | \GV, \GR)}{\partial \phi[v]} &= \sum_{(\GV, \GR, v) \in \mathcal{T}} \left[\frac{\partial \Gamma(\GV, \GR, v)}{\partial \phi[v]} - \sum_u P(u|\GV, \GR) \frac{\partial\Gamma(\GV, \GR, u)}{\partial \phi[v]}\right]\\
    &= \textcolor{red}{\sum_{(\GV, \GR) \in \mathcal{N}^1_-[v]} g(\GR) \odot \phi[\GV]\left(1 - P(v|\GV, \GR)\right)}.
    \end{align*}
\item Components corresponding to the triplets where $\GV \neq v \wedge \GW \neq v$. The sum of these components is given by:
    \begin{align*}
    \sum_{(\GV, \GR, \GW) \in \mathcal{T}} \frac{\partial \log P(\GW | \GV, \GR)}{\partial \phi[v]} &=
    \sum_{(\GV, \GR, \GW) \in \mathcal{T}} \left[0 - \sum_{u} P(u|\GV, \GR) \frac{\partial \Gamma(\GV, \GR, u)}{\partial \phi[v]}\right] \\ 
    &= \textcolor{blue}{\sum_{(\GV, \GR, \GW) \in \mathcal{T}} - P(v|\GV, \GR) \frac{\partial \Gamma(\GV, \GR, v)}{\partial \phi[v]}}. \\
    &= \textcolor{blue}{\sum_{(\GV, \GR, \GW) \in \mathcal{T}} - P(v|\GV, \GR) g(\GR) \odot \phi[\GV]}.
    \end{align*}
\end{enumerate}
Collecting these three categories, the GD operator over $\phi[v]$, or rather the node representation update in DistMult, can be rewritten as:
\begin{align}
\label{eq:node_view}
\GD(\phi, \mathcal{T})[v] = \phi[v] + \alpha
\textcolor{red}{\underbrace{\sum_{\{(\GR, \GW) \in \mathcal{N}_{+}^1[v] \}} \phi[\GW] \odot  g(\GR)
+ \sum_{(\GR, \GV) \in \mathcal{N}^1_-[v]} \phi[\GV] \odot g(\GR) \left(1 - P(v|\GV, \GR)\right)}_{v\text{'s neighbourhood} \to v}} \\
\textcolor{blue}{\underbrace{-\alpha\sum_{(\GV, \GR, \GW) \in \mathcal{T}, \GV \neq v, \GW \neq v}P(v|\GV, \GR) g(\GR) \odot \phi[\GV]
-  \alpha\sum_{(v, \GR, \GW) \in \mathcal{T}} \sum_u P(u|v, \GR)g(\GR) \odot \phi[u]}_{\text{beyond neighbourhood} \to v}}.
\end{align}
Note that the component ``\textcolor{red}{$v\text{'s neighbourhood} \to v$}''~(highlighted in red) in \cref{eq:node_view} is a sum over $v$'s neighbourhood -- gathering information from positive neighbours $\phi[\GW], (\cdot, \GW) \in \mathcal{N}_+^1[v]$ and negative neighbours $\phi[\GV], (\cdot, \GV) \in \mathcal{N}_-^1[v]$.
Hence, each atomic term of the sum can be seen as a message vector between $v$ and $v$'s neighbouring node.
Formally, letting $w$ be $v$'s neighbouring node, the message vector can be written as follows
\begin{equation}
\begin{split}
m[v, r, w] = q_{\mathrm{M}}(\phi[v], r, \phi[w]) = 
\begin{cases}
\phi[w] \odot  g(r), \text{ if } (r,w) \in \mathcal{N}_{+}^1[v], \\
\phi[w] \odot g(r) (1 - P(v|w, r)), \text{ if } (r,w) \in \mathcal{N}_-^1[v],\\
\end{cases}
\end{split}
\end{equation}
which induces a bi-directional message function $q_M$.
On the other hand, the summation over these atomic terms~(message vectors) induces the aggregate function $q_{\mathrm{A}}$:
\begin{equation}
\label{eq:aggr}
\begin{split}
z[v] &= q_{\mathrm{A}}(\{m[v, r, w]\, :\, (r,w) \in \mathcal{N}^1[v]\}) \\  
     &= \sum_{(\GR,\GW) \in \mathcal{N}_+^1[v]} m^l[v,\GR,\GW] + \sum_{(\GR, \GV) \in \mathcal{N}_-^1[v]} m^l[\GV,\GR, v] = \sum_{(r,w) \in \mathcal{N}^1[v]} m[v,r,w]. \\ 
\end{split}
\end{equation}
Finally, the component ``\textcolor{blue}{$\text{beyond neighbourhood} \to v$}'' (highlighted in blue) is a term that contains dynamic information flow from global nodes to $v$.
If we define: $$n[v] = \frac{1}{|\mathcal{T}|} \sum_{(v, \GR, \GW) \in \mathcal{T}} \sum_u P(u|v, \GR)g(\GR) \odot \phi[u]
+ \frac{1}{|\mathcal{T}|}\sum_{(\GV, \GR, \GW) \in \mathcal{T}, \GV \neq v, \GW \neq v}P(v|\GV, \GR) g(\GR) \odot \phi[\GV],$$
the GD operator over $\phi[v]$ then boils down to an update function which utilises previous node state $\phi[v]$, aggregated message $z[v]$ and a global term $n[v]$ to produce the new node state:
\begin{equation}
\label{eq:update}
\GD(\phi, \mathcal{T})[v] = q_{\mathrm{U}}(\phi[v], z[v]) = \phi[v] + \alpha z[v] - \beta n[v].
\end{equation}
Furthermore, $n[v]$ can be seen as a weighted sum of expectations by recasting the summations over triplets as expectations:
\begin{equation}
\begin{split}
n[v] = \frac{|\mathcal{N}_+^1[v]|}{|\mathcal{T}|}\mathbb{E}_{(v, \GR, \GW) \sim P_{\mathcal{N}_+^1[v]}} \mathbb{E}_{u \sim P(\cdot|v, \GR)} g(\GR) \odot \phi[u]
+ \frac{|\mathcal{T}^{-v}|}{|\mathcal{T}|}\mathbb{E}_{(\GV, \GR, \GW) \sim P_{\mathcal{T}^{-v}}}P(v|\GV, \GR, ) g(\GR) \odot \phi[\GV] \\
\end{split}
\end{equation}
where $\mathcal{T}^{-v} = 
\{(\GV, \GR, \GV') \in \mathcal{T} | \GV \neq v \land \GV' \neq v\}$ is the set of triplets that do not contain $v$.
\end{proof}
\subsection{Extension to AdaGrad and N3 Regularisation}
State-of-the-art FMs are often trained with training strategies adapted for each model category. For example, using an N3 regularizer~\cite{lacroix2018canonical} and AdaGrad optimiser~\cite{JMLR:v12:duchi11a}, which we use for our experiments.
For N3 regularizer, we add a gradient term induced by the regularised loss:
$$\frac{\partial L}{\partial \phi[v]} = \frac{\partial L_{\text{fit}}}{\partial \phi[v]} + \lambda \frac{\partial L_{\text{reg}}}{\partial \phi[v]} = \frac{\partial L_{\text{fit}}}{\partial \phi[v]} + \lambda \text{sign}(\phi[v]) \phi[v]^2$$

where $L_{\text{fit}}$ is the training loss, $L_{\text{reg}}$ is the regularisation term, $\text{sign}(\cdot)$ is a element-wise sign function, and $\lambda \in \mathbb{R}_{+}$ is a hyper-parameter specifying the regularisation strength. The added component relative to this regularizer fits into the message function as follows:
\begin{equation}
\label{eq:message}
\begin{split}
q_{\mathrm{M}}(\phi[v], r, \phi[w]) = 
\begin{cases}
\phi[w] \odot  g(r) - \lambda \text{sign}(\phi[w]) \phi[w]^2, \text{ if } (r,w) \in \mathcal{N}_{+}^1[v], \\
\phi[w] \odot g(r) (1 - P(v|w, r)) - \lambda \text{sign}(\phi[w]) \phi[w]^2 , \text{ if } (w, r) \in \mathcal{N}_-^1[v];\\
\end{cases}
\end{split}
\end{equation}

Our derivation in \cref{sec:fm_sgd} focuses on~(stochastic) gradient descent as the optimiser for training FMs. Going beyond this, complex gradient-based optimisers like AdaGrad use running statistics of the gradients. For example, for an AdaGrad optimiser, the gradient is element-wisely re-scaled by $\frac{1}{\sqrt{s_v} + \epsilon} \nabla_{\phi[v]} L$
where $s$ is the running sum of squared gradients and $\epsilon>0$ is a hyper-parameter added to the denominator to improve numerical stability. Such re-scaling can be absorbed into the update equation:
$$
\text{AdaGrad}(\phi, \mathcal{T})[v] = \phi[v] + (\alpha z[v] - \beta n[v]) * \frac{1}{\sqrt{s[v]} + \epsilon}.
$$
In general, we can interpret any auxiliary variable introduced by the optimizer (e.g. the velocity) as an additional part of the entities and relations representations on which message passing happens. However, the specific equations would depend on the optimizer’s dynamics and would be hard to formally generalise. 

\subsection{Extensions to Other Score Functions e.g. ComplEx}
The two main design choices in \cref{the:mp_fm_proof} are 1) the score function $\Gamma$, and  2) the optimization dynamics over the node embeddings.
In the paper, we chose DistMult and GD because of their mathematical simplicity, leading to easier-to-read formulas.
We can adapt the theorem to general, smooth scoring functions $\Gamma:.\mathcal{E} \times \mathcal{R} \times \mathcal{E} \to \mathbf{R}$ by replacing occurrences of the gradient of DistMult with a generic $\nabla \Gamma$~(the gradient of DistMult w.r.t. $\phi[v]$ at $(v, r, w)$  is simply $g(r) \odot \phi[w]$). 
This gives us the following lemma:

\begin{lemma}[Message passing in FMs] 
The gradient descent operator $\GD$ (\ref{eq:fm_node_view}) on the node embeddings of a general score function with the maximum likelihood objective in \cref{eq:objective} and a multi-relational graph $\mathcal{T}$ defined over entities $\mathcal{E}$ induces a message-passing operator whose composing functions are:  
\begin{align} 
 &q_{\mathrm{M}}(\phi[v], r, \phi[w]) = 
 \left\lbrace
\begin{array}{lr}
\nabla_{\phi[v]} \Gamma(v, r, w) & \text{if} \; (r,w) \in \mathcal{N}_{+}^1[v], \\
(1 - P_\theta (v|w, r)) 
\nabla_{\phi[v]} \Gamma(w, r, v)  & \text{if} \; (r, w) \in \mathcal{N}_-^1[v];
\end{array}
\right.
\\
&q_{\mathrm{A}}(\{m[v, r, w]\, :\, (r,w) \in \mathcal{N}^1[v]\}) = \sum_{(r,w) \in \mathcal{N}^1[v]} m[v,r,w]; \\
&q_{\mathrm{U}}(\phi[v], z[v]) = \phi[v] + \alpha z[v] - \beta 
n[v],
\end{align}
where, defining the sets of triplets $\mathcal{T}^{-v} =\{(s, r, o) \in \mathcal{T} \; : \; s\neq v \wedge o\neq v\}$, 
\begin{equation}
n[v]= 
\frac{|\mathcal{N}_{+}^{1}[v]|}{|\mathcal{T}|}
\mathbb{E}_{
P_{\mathcal{N}_+^{1}[v]}} \mathbb{E}_{u \sim P_{\theta}(\cdot|v, r)} \nabla_{\phi[v]} \Gamma(v, r, u) 
+
\frac{|\mathcal{T}^{-v}|}{|\mathcal{T}|}
\mathbb{E}_{ P_{\mathcal{T}^{-v}}}P_\theta(v|s, r) \nabla_{\phi[v]} \Gamma(s, r, v) ,
\end{equation}
where $P_{\mathcal{N}^{1}_+[v]}$ and $P_{\mathcal{T}^{-v}}$ are the empirical probability distributions associated to the respective sets.
\end{lemma}

Accordingly, the node representation updating equations in \cref{sec:edge_view} can be re-written as follows

\begin{align*}
 \GD(\phi, \{(v, r, w)\})[v] = \phi[v] + \alpha\left(
 \textcolor{black}{
    \underbrace{
    \nabla_{\phi[v]} \Gamma(v, r, w)}_{w \to v}} 
 \textcolor{black}{
    \underbrace{- \sum_{u\in \mathcal{E}}P_\theta(u|v,r) \nabla_{\phi[v]} \Gamma(v, r, u)}_{u \to v}}
 \right),
\end{align*}

\begin{align*}
\GD(\phi, \{(v, r, w)\})[w]
= \phi[w] + \alpha \textcolor{black}{\underbrace{\left(1 - P_\theta(w|v,r)\right) \nabla_{\phi[w]} \Gamma(v, r, w)}_{v \to w}},
\end{align*}

\begin{align*}
\GD(\phi, \{(v, r, w)\})[u] = \phi[u] + \alpha \left(\textcolor{black}{\underbrace{-P_\theta(u|v,r)\nabla_{\phi[u]} \Gamma(v, r, u)}_{v \to u}}
\right). \\
\end{align*}

$\nabla_{\phi[\cdot]} \Gamma$ can be different for different models.
For example, here we offer a specific derivation for ComplEx~\citep{pmlr-v48-trouillon16}. 
Let $d=K/2$ be the hidden size for ComplEx. The ComplEx score function is given as follows
\begin{equation}
\begin{split}
\Gamma(v, r, w) &= <\psi[r]_{(0:d)}, \phi[v]_{(0:d)}, \phi[w]_{(0:d)}> + <\psi[r]_{(0:d)}, \phi[v]_{(d:)}, \phi[w]_{(d:)}> \\ 
&+ <\psi[r]_{(d:)}, \phi[v]_{(0:d)}, \phi[w]_{(d:)}> - <\psi[r]_{(d:)}, \phi[v]_{(d:)}, \phi[w]_{(0:d)}>    
\end{split}
\end{equation}
where $(0:d)$ indicates the real part of the complex vector and $(d:)$ indicates the image part of the complex vector.
The gradients of the ComplEx score function with respect to the real/image node representations are given by 
$\frac{\partial \Gamma(v, r, w)}{\partial \phi[v]_{(0:d)}} = \psi[r]_{(0:d)} \odot \phi[w]_{(0:d)} + \psi[r]_{(d:)} \odot \phi[w]_{(d:)},$
$\frac{\partial \Gamma(v, r, w)}{\partial \phi[v]_{(d:)}} = \psi[r]_{(0:d)} \odot \phi[w]_{(d:)} - \psi[r]_{(d:)} \odot \phi[w]_{(0:d)},$
$\frac{\partial \Gamma(v, r, w)}{\partial \phi[w]_{(0:d)}} = \psi[r]_{(0:d)} \odot \phi[v]_{(0:d)} - \psi[r]_{(d:)} \odot \phi[v]_{(d:)},$
$\frac{\partial \Gamma(v, r, w)}{\partial \phi[w]_{(d:)}} = \psi[r]_{(0:d)} \odot \phi[v]_{(d:)} + \psi[r]_{(d:)} \odot \phi[v]_{(0:d)}.$
Concatenating gradients for the real part and the image part, we have the gradients
$$\nabla_{\phi[v]} \Gamma(v, r, w) = \frac{\partial \Gamma(v, r, w)}{\partial \phi[v]_{(0:d)}} \mathbin\Vert \frac{\partial \Gamma(v, r, w)}{\partial \phi[v]_{(d:)}},$$
$$\nabla_{\phi[w]} \Gamma(v, r, w) = \frac{\partial \Gamma(v, r, w)}{\partial \phi[w]_{(0:d)}} \mathbin\Vert \frac{\partial \Gamma(v, r, w)}{\partial \phi[w]_{(d:)}}.$$

\section{Additional Results on Inductive KGC Tasks}
In this paper, we describe the results on FB15K237\_v1\_ind under some random seed.
To confirm the significance and sensitivity, we further experiment with additional 5 random seeds.
Due to our computational budget, for this experiment, we resorted to a coarse grid when performing the hyper-parameters sweeps. 
Following standard evaluation protocols, we report the mean values and standard deviations of the filtered Hits@10 over 5 random seeds. Numbers for Neural-LP, DRUM, RuleN, GraIL, and NBFNet are taken from the literature~\cite{Teru2020InductiveRP, zhaochengzhu2021}. ``-'' means the numbers are not applicable. \cref{tab:ilp_all_50} summarises the results.
\factorgnns are able to make use of both types of input features, while textual features benefit both GAT and \factorgnns for most datasets. Increasing depth benefits WN18RR\_v$i$\_ind ($i \in [1, 2, 3, 4]$) most. Future work could consider the impact of textual node features provided by different pre-trained language models. Another interesting direction is to investigate the impact of depth on GNNs for datasets like WN18RR, where many kinds of hierarchies are observed in the data.

In addition to the \emph{partial ranking} evaluation protocol, where the ground-truth subject/object entity is ranked against 50 sampled entities,\footnote{One implementation for such evaluation can be found in \href{https://github.com/kkteru/grail/blob/master/test_ranking.py\#L448}{GraIL's codebase}.} we also consider the \emph{full ranking} evaluation protocol, where the ground-truth subject/object entity is ranked against all the entities. ~\cref{tab:ilp_all_full} summarises the results. Empirically, we observe that \emph{full ranking} is more suitable for reflecting the differences between models than \emph{partial ranking}. It also has less variance than \emph{partial ranking}, since it requires no sampling from the candidate entities. Hence, we believe there is good reason to recommend the community to use \emph{full ranking} for these datasets in the future.

\afterpage{\clearpage
\begin{landscape}
\begin{table}[tbp]
\caption{Hits@10 with Partial Ranking against 50 Negative Samples. ``[T]'' indicates using textual encodings of entity descriptions~\cite{reimers-2019-sentence-bert} as input (positional) node features; ``[R]'' indicates using frozen random vectors as input (positional) node feature.}
\label{tab:ilp_all_50}
\resizebox{\columnwidth}{!}{%
\begin{tabular}{llllllllllllll}
\hline & \multicolumn{4}{c}{ WN18RR } & \multicolumn{4}{c}{ FB15k-237 } & & \multicolumn{3}{c}{ NELL-995 } \\
& v1 & v2 & v3 & v4 & v1 & v2 & v3 & v4 & v1 & v2 & v3 & v4 \\
\hline 
No Pretrain [R] &  $0.220\tiny{\pm0.048}$ &  $0.226\tiny{\pm0.013}$ &  $0.244\tiny{\pm0.020}$ &  $0.218\tiny{\pm0.050}$ &  $0.215\tiny{\pm0.019}$ &  $0.207\tiny{\pm0.008}$ &  $0.211\tiny{\pm0.002}$ &  $0.205\tiny{\pm0.008}$ &  $0.543\tiny{\pm0.022}$ &  $0.207\tiny{\pm0.008}$ &  $0.216\tiny{\pm0.004}$ &  $0.198\tiny{\pm0.006}$ \\
No Pretrain [T] & $0.267\tiny{\pm0.020}$ &  $0.236\tiny{\pm0.020}$ &  $0.292\tiny{\pm0.025}$ &  $0.253\tiny{\pm0.022}$ &  $0.242\tiny{\pm0.018}$ &  $0.227\tiny{\pm0.007}$ &  $0.240\tiny{\pm0.011}$ &  $0.244\tiny{\pm0.003}$ &  $0.538\tiny{\pm0.079}$ &  $0.234\tiny{\pm0.017}$ &  $0.242\tiny{\pm0.020}$ &  $0.191\tiny{\pm0.036}$ \\
\hline
Neural-LP & $0.744$ & $0.689$ & $0.462$ & $0.671$ & $0.529$ & $0.589$ & $0.529$ & $0.559$ & $0.408$ & $0.787$ & $0.827$ & $0.806$ \\
DRUM & $0.744$ & $0.689$ & $0.462$ & $0.671$ & $0.529$ & $0.587$ & $0.529$ & $0.559$ & $0.194$ & $0.786$ & $0.827$ & $0.8 06$ \\
RuleN & $0.809$ & $0.782$ & $0.534$ & $0.716$ & $0.498$ & $0.778$ & $0.877$ & $0.856$ & $0.535$ & $0.818$ & $0.773$ & $0.614$ \\
\hline
GAT($3$) [R] &  $0.583\tiny{\pm0.022}$ &  $0.797\tiny{\pm0.002}$ &  $0.560\tiny{\pm0.005}$ &  $0.660\tiny{\pm0.015}$ &  $0.333\tiny{\pm0.042}$ &  $0.312\tiny{\pm0.036}$ &  $0.407\tiny{\pm0.072}$ &  $0.363\tiny{\pm0.050}$ &  $0.906\tiny{\pm0.004}$ &  $0.303\tiny{\pm0.031}$ &  $0.351\tiny{\pm0.009}$ &  $0.187\tiny{\pm0.098}$ \\
GAT($6$) [R] &  $0.850\tiny{\pm0.014}$ &  $0.841\tiny{\pm0.001}$ &  $0.631\tiny{\pm0.020}$ &  $0.802\tiny{\pm0.004}$ &  $0.401\tiny{\pm0.020}$ &  $0.445\tiny{\pm0.018}$ &  $0.461\tiny{\pm0.048}$ &  $0.406\tiny{\pm0.143}$ &  $0.811\tiny{\pm0.039}$ &  $0.670\tiny{\pm0.055}$ &  $0.341\tiny{\pm0.042}$ &  $0.301\tiny{\pm0.002}$ \\
GAT($3$) [T] &  $\mathbf{0.970\tiny{\pm0.002}}$ &  $0.980\tiny{\pm0.001}$ &  $0.897\tiny{\pm0.005}$ &  $0.960\tiny{\pm0.001}$ &  $0.806\tiny{\pm0.003}$ &  $0.942\tiny{\pm0.001}$ &  $0.941\tiny{\pm0.002}$ &  $0.954\tiny{\pm0.001}$ &  $0.938\tiny{\pm0.005}$ &  $0.839\tiny{\pm0.001}$ &  $0.962\tiny{\pm0.001}$ &  $0.354\tiny{\pm0.002}$ \\
GAT($6$) [T] &  $0.965\tiny{\pm0.002}$ &  $0.986\tiny{\pm0.001}$ &  $0.920\tiny{\pm0.002}$ &  $0.970\tiny{\pm0.003}$ &  $0.826\tiny{\pm0.004}$ &  $0.943\tiny{\pm0.001}$ &  $0.927\tiny{\pm0.003}$ &  $0.927\tiny{\pm0.001}$ &  $0.904\tiny{\pm0.000}$ &  $0.811\tiny{\pm0.001}$ &  $0.880\tiny{\pm0.001}$ &  $0.297\tiny{\pm0.003}$ \\
GraIL & $0.825$ & $0.787$ & $0.584$ & $0.734$ & $0.642$ & $0.818$ & $0.828$ & $0.893$ & $0.595$ & $0.933$ & $0.914$ & $0.732$ \\
NBFNet & $0.948$ & $0.905$ & $0.893$ & $0.890$ & $0.834$ & $0.949$ & $0.951$ & $0.960$ & - & - & - & - \\
\hline
ReFactorGNN($3$) [R] &  $0.899\tiny{\pm0.003}$ &  $0.842\tiny{\pm0.004}$ &  $0.605\tiny{\pm0.000}$ &  $0.801\tiny{\pm0.002}$ &  $0.673\tiny{\pm0.000}$ &  $0.812\tiny{\pm0.002}$ &  $0.833\tiny{\pm0.003}$ &  $0.877\tiny{\pm0.002}$ &  $0.913\tiny{\pm0.000}$ &  $0.913\tiny{\pm0.011}$ &  $0.893\tiny{\pm0.000}$ &  $0.838\tiny{\pm0.002}$ \\
ReFactorGNN($6$) [R] &  $0.885\tiny{\pm0.000}$ &  $0.854\tiny{\pm0.003}$ &  $0.738\tiny{\pm0.006}$ &  $0.817\tiny{\pm0.004}$ &  $0.787\tiny{\pm0.007}$ &  $0.903\tiny{\pm0.003}$ &  $0.903\tiny{\pm0.002}$ &  $0.920\tiny{\pm0.002}$ &  $0.971\tiny{\pm0.007}$ &  $0.957\tiny{\pm0.003}$ &  $0.935\tiny{\pm0.003}$ &  $0.927\tiny{\pm0.001}$ \\
ReFactorGNN($3$) [T] &  $0.918\tiny{\pm0.002}$ &  $0.973\tiny{\pm0.001}$ &  $0.910\tiny{\pm0.003}$ &  $0.934\tiny{\pm0.001}$ &  $0.900\tiny{\pm0.004}$ &  $0.959\tiny{\pm0.001}$ &  $0.952\tiny{\pm0.002}$ &  $0.968\tiny{\pm0.001}$ &  $\mathbf{0.955\tiny{\pm0.004}}$ &  $0.931\tiny{\pm0.001}$ &  $0.978\tiny{\pm0.001}$ &  $0.929\tiny{\pm0.001}$ \\
ReFactorGNN($6$) [T] &  $\mathbf{0.970\tiny{\pm0.002}}$ &  $\mathbf{0.988\tiny{\pm0.001}}$ &  $\mathbf{0.944\tiny{\pm0.002}}$ &  $\mathbf{0.987\tiny{\pm0.000}}$ &  $\mathbf{0.920\tiny{\pm0.001}}$ &  $\mathbf{0.963\tiny{\pm0.001}}$ &  $\mathbf{0.962\tiny{\pm0.002}}$ &  $\mathbf{0.970\tiny{\pm0.002}}$ &  $0.949\tiny{\pm0.011}$ &  $\mathbf{0.963\tiny{\pm0.001}}$ &  $\mathbf{0.994\tiny{\pm0.000}}$ &  $\mathbf{0.955\tiny{\pm0.002}}$ \\
\bottomrule
\end{tabular}
}
\end{table}
\begin{table}[tbp]
\caption{Hits@10 with Full Ranking against All Candidate Entities. ``[T]'' indicates using textual encodings of entity descriptions~\cite{reimers-2019-sentence-bert} as input (positional) node features; ``[R]'' indicates using frozen random vectors as input (positional) node feature.}
\label{tab:ilp_all_full}
\resizebox{\columnwidth}{!}{%
\begin{tabular}{llllllllllllll}
\hline & \multicolumn{4}{c}{ WN18RR } & \multicolumn{4}{c}{ FB15k-237 } & & \multicolumn{3}{c}{ NELL-995 } \\
& v1 & v2 & v3 & v4 & v1 & v2 & v3 & v4 & v1 & v2 & v3 & v4 \\
\hline 
No Pretrain [R] &  $0.020\tiny{\pm0.006}$ &  $0.004\tiny{\pm0.001}$ &  $0.004\tiny{\pm0.003}$ &  $0.003\tiny{\pm0.001}$ &  $0.013\tiny{\pm0.003}$ &  $0.012\tiny{\pm0.001}$ &  $0.004\tiny{\pm0.001}$ &  $0.002\tiny{\pm0.001}$ &  $0.255\tiny{\pm0.021}$ &  $0.004\tiny{\pm0.001}$ &  $0.001\tiny{\pm0.001}$ &  $0.003\tiny{\pm0.001}$ \\
No Pretrain [T] &  $0.027\tiny{\pm0.009}$ &  $0.007\tiny{\pm0.003}$ &  $0.006\tiny{\pm0.001}$ &  $0.005\tiny{\pm0.001}$ &  $0.014\tiny{\pm0.001}$ &  $0.010\tiny{\pm0.001}$ &  $0.007\tiny{\pm0.001}$ &  $0.006\tiny{\pm0.001}$ &  $0.262\tiny{\pm0.031}$ &  $0.006\tiny{\pm0.002}$ &  $0.006\tiny{\pm0.002}$ &  $0.003\tiny{\pm0.001}$ \\
\hline
GAT($3$) [R] &  $0.171\tiny{\pm0.008}$ &  $0.504\tiny{\pm0.026}$ &  $0.260\tiny{\pm0.022}$ &  $0.089\tiny{\pm0.017}$ &  $0.074\tiny{\pm0.003}$ &  $0.050\tiny{\pm0.014}$ &  $0.051\tiny{\pm0.019}$ &  $0.023\tiny{\pm0.012}$ &  $0.806\tiny{\pm0.019}$ &  $0.003\tiny{\pm0.002}$ &  $0.008\tiny{\pm0.007}$ &  $0.008\tiny{\pm0.004}$ \\
GAT($6$) [R] &  $0.575\tiny{\pm0.005}$ &  $0.698\tiny{\pm0.003}$ &  $0.312\tiny{\pm0.000}$ &  $0.606\tiny{\pm0.002}$ &  $0.048\tiny{\pm0.004}$ &  $0.028\tiny{\pm0.004}$ &  $0.033\tiny{\pm0.018}$ &  $0.015\tiny{\pm0.026}$ &  $0.491\tiny{\pm0.112}$ &  $0.110\tiny{\pm0.048}$ &  $0.031\tiny{\pm0.010}$ &  $0.031\tiny{\pm0.002}$ \\
GAT($3$) [T] &  $0.794\tiny{\pm0.000}$ &  $0.826\tiny{\pm0.000}$ &  $0.468\tiny{\pm0.000}$ &  $0.705\tiny{\pm0.000}$ &  $0.331\tiny{\pm0.000}$ &  $0.585\tiny{\pm0.000}$ &  $0.505\tiny{\pm0.000}$ &  $0.449\tiny{\pm0.000}$ &  $0.856\tiny{\pm0.000}$ &  $0.245\tiny{\pm0.000}$ &  $0.345\tiny{\pm0.000}$ &  $0.078\tiny{\pm0.000}$ \\
GAT($6$) [T] &  $0.815\tiny{\pm0.000}$ &  $0.808\tiny{\pm0.000}$ &  $0.469\tiny{\pm0.000}$ &  $0.701\tiny{\pm0.000}$ &  $0.416\tiny{\pm0.000}$ &  $0.483\tiny{\pm0.000}$ &  $0.391\tiny{\pm0.000}$ &  $0.388\tiny{\pm0.000}$ &  $0.851\tiny{\pm0.000}$ &  $0.189\tiny{\pm0.000}$ &  $0.137\tiny{\pm0.000}$ &  $0.023\tiny{\pm0.000}$ \\
\hline
ReFactorGNN($3$) [R] &  $0.826\tiny{\pm0.000}$ &  $0.758\tiny{\pm0.002}$ &  $0.374\tiny{\pm0.004}$ &  $0.707\tiny{\pm0.000}$ &  $0.455\tiny{\pm0.010}$ &  $0.603\tiny{\pm0.008}$ &  $0.556\tiny{\pm0.003}$ &  $0.587\tiny{\pm0.003}$ &  $0.907\tiny{\pm0.004}$ & $0.700\tiny{\pm0.001}$ &  $0.630\tiny{\pm0.001}$ &  $0.511\tiny{\pm0.001}$ \\
ReFactorGNN($6$) [R] &  $0.826\tiny{\pm0.001}$ &  $0.769\tiny{\pm0.005}$ &  $0.440\tiny{\pm0.001}$ &  $0.731\tiny{\pm0.000}$ &  $0.558\tiny{\pm0.007}$ &  $0.694\tiny{\pm0.006}$ &  $0.639\tiny{\pm0.006}$ &   $0.640\tiny{\pm0.000}$ &  $\mathbf{0.967\tiny{\pm0.005}}$ &  $\mathbf{0.764\tiny{\pm0.009}}$ &  $0.697\tiny{\pm0.005}$ &  $\mathbf{0.703\tiny{\pm0.001}}$ \\
ReFactorGNN($3$) [T] &  $0.805\tiny{\pm0.000}$ &  $0.796\tiny{\pm0.003}$ &  $0.483\tiny{\pm0.000}$ &  $0.682\tiny{\pm0.000}$ &  $0.589\tiny{\pm0.001}$ &  $0.672\tiny{\pm0.001}$ &  $0.610\tiny{\pm0.001}$ &  $0.611\tiny{\pm0.001}$ &  $0.918\tiny{\pm0.000}$ &  $0.629\tiny{\pm0.001}$ &  $0.634\tiny{\pm0.000}$ &  $0.305\tiny{\pm0.000}$ \\
ReFactorGNN($6$) [T] &  $\mathbf{0.844\tiny{\pm0.004}}$ &  $\mathbf{0.848\tiny{\pm0.003}}$ &  $\mathbf{0.522\tiny{\pm0.001}}$ &  $\mathbf{0.781\tiny{\pm0.001}}$ &  $\mathbf{0.619\tiny{\pm0.000}}$ &  $\mathbf{0.721\tiny{\pm0.001}}$ &  $\mathbf{0.663\tiny{\pm0.000}}$ &   $\mathbf{0.660\tiny{\pm0.000}}$ &  $0.913\tiny{\pm0.000}$ &  $0.733\tiny{\pm0.000}$ &  $\mathbf{0.711\tiny{\pm0.000}}$ &  $0.417\tiny{\pm0.000}$ \\
\bottomrule
\end{tabular}
}
\end{table}
\end{landscape}
}

\section{Additional Results on The Impact of Meaningful Node Features}
To better understand the impact that meaningful node features have on \factorgnns for the task of knowledge graph completion, we compare \factorgnns trained with RoBERTa Encodings~(one example of meaningful node features) and \factorgnns trained with Random Vectors~(not meaningful node features).
We perform experiments on \textit{FB15K237\_v1} and vary the number of message-passing layers: $L \in \{3, 6, \infty\}$.
\cref{tab:diff_feature} summarises the differences.
We can see that meaningful node features are highly beneficial if \factorgnns are only provided with a small number of message-passing layers.
As more message-passing layers are allowed, the benefit of \factorgnns diminishes. 
The extreme case would be $L= \infty$, where the benefit of meaningful node features becomes negligible. 
We hypothesise that this might be why meaningful node features haven not been found to be useful for transductive knowledge graph completion.

\begin{table}[tbp]
    \centering
    \begin{tabular}{cccc}
    \toprule
       Depth & 3 & 6 & $\infty$  \\
    $\Delta$ Test MRR & 0.060 & 0.045 & 0.016 \\
    \bottomrule
    \end{tabular}
    \caption{The Impact of Meaningful Node Feature on \textit{FB15K237\_v1}. $\Delta$ Test MRR is computed by \texttt{test mrr (textual node features)} $-$ \texttt{test mrr (random node features)}. Larger $\Delta$ means meaningful node features bring more benefit. }
    \label{tab:diff_feature}
\end{table}

\section{Additional Results on Parameter Efficiency} ~\cref{fig:parameter_efficiency_fb237v2} shows the parameter efficiency on the dataset \textit{FB15K237\_v2}.
\begin{figure}
    \centering
    \includegraphics[scale=0.3]{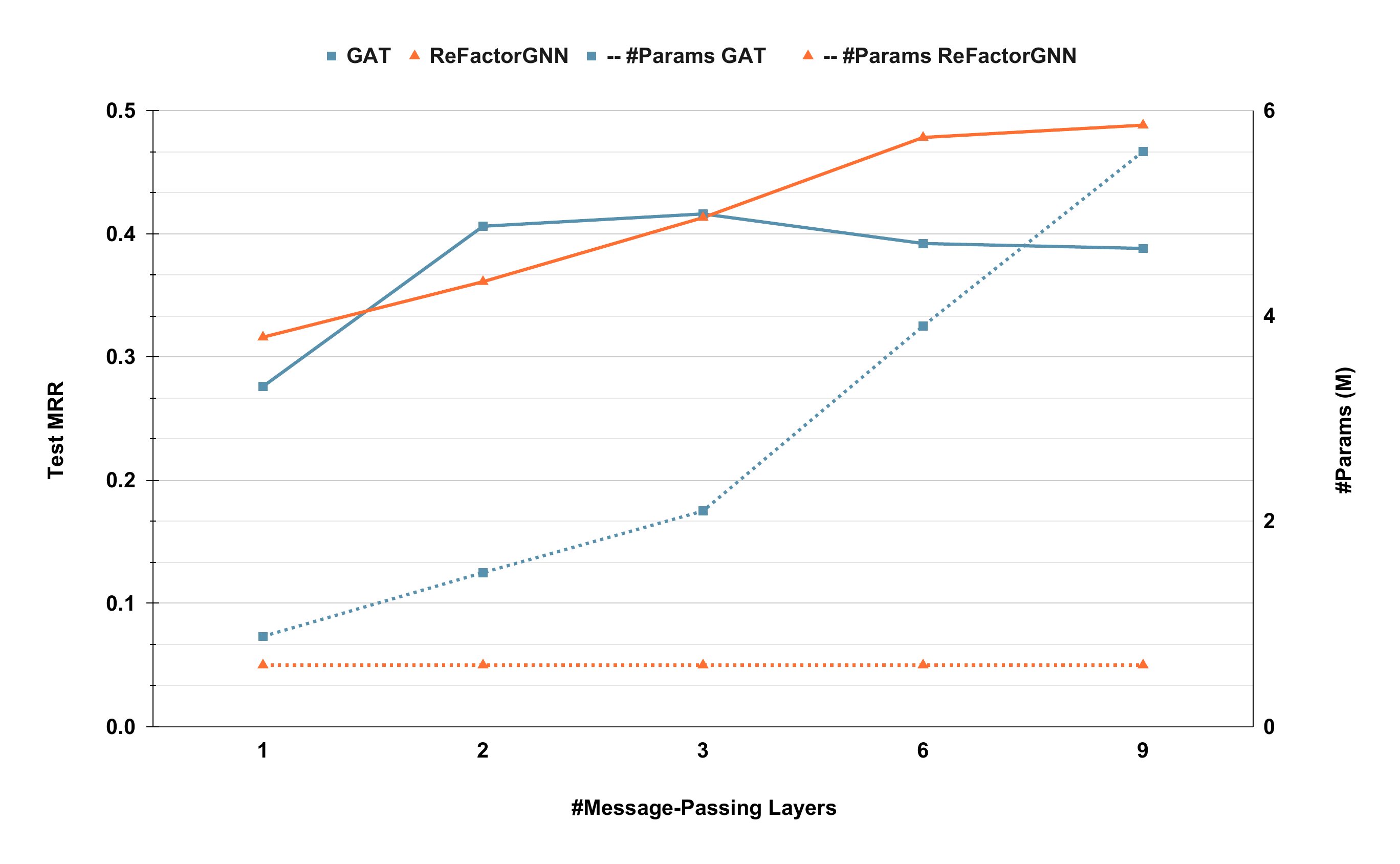}
    \caption{Performance vs Parameter Efficiency as \#Layers Increases on \textit{FB15K237\_v2}. The left axis is Test MRR while the right axis is \#Parameters. The solid lines and dashed lines indicate the changes of Test MRR and the changes of \#Parameters.}     
    \label{fig:parameter_efficiency_fb237v2}
\end{figure}

\section{Discussion on Complexity}
We can analyse the scalability of \factorgnns along three axes, the number of layers $L$, the embedding size $d$, and the number of triplets/edges in the graph $|\mathcal{T}|$.
For scalability w.r.t. to the number of layers, let $L$ denote the number of message-passing layers. Since \factorgnns tie the weights across the layers, the parameter complexity of \factorgnns is $\mathcal{O}(1)$, while it is $\mathcal{O}(L)$ for standard GNNs such as GATs, GCNs, and R-GCNs. Additionally, since \factorgnns adopt layer-wise training enabled via the external memory for node state caching, the training memory footprint is also $\mathcal{O}(1)$ as opposed to $\mathcal{O}(L)$ for standard GNNs. 
For scalability w.r.t the embedding size, let $d$ denote the embedding size. \factorgnns scale linearly with $d$, as opposed to most GNNs in literature where the parameter and time complexities scale quadratically with $d$.
For scalability w.r.t. the number of triplets/edges in the graph, we denote the entity set as $\mathcal{E}$, the relation set as $\mathcal{R}$, and the triplets as $\mathcal{T}$. NBFNet requires $O(LT^2d + L\mathcal{T}Vd^2)$ inference run-time complexity since the message-passing is done for every source node and query relation -- quadratic w.r.t the number of triplets $\mathcal{T}$ while \factorgnns are of linear complexity w.r.t $\mathcal{T}$. 
Extending the complexity analysis in NBFNet~\citep{zhaochengzhu2021} to all the triplets, we include a detailed table for complexity comparison in \cref{tab:complexity}. The inference complexity refers to the cost per forward pass over the entire graph.

\begin{table}[h]
    \centering
\resizebox{\columnwidth}{!}{%
    \begin{tabular}{r|ccccc}
    \toprule
         &  Parameter & Training Memory  & Inference Memory & Training Time & Inference Time\\
         &  Complexity & Complexity & Complexity & Complexity & Complexity\\
    \hline
 GAT     & $O(Ld^2)$    & $O(L|V|d)$ & $O(L|V|d)$ & $O(L|V|d^2 + L|T|d)$   & $O(L|V|d^2 + L|T|d)$ \\
 R-GCN    & $O(L|R|d^2)$ & $O(L|V|d)$ & $O(L|V|d)$ & $O(L|T|d^2 + L|V|d^2)$ & $O(L|T|d^2 + L|V|d^2)$ \\
 NBFNet  & $O(L|R|d^2)$ & $O(L|V||T|d)$ & $O(L|V||T|d)$ & $O(L|T|^2d + L |T||V|d^2)$ & $O(L|T|^2d + L |T||V|d^2)$ \\                    
\factorgnns & $O(|R|d)$ & $O(|V|d)$  & $O(L|V|d)$ & $O(|T||V|d)$ & $O(L|T||V|d)$  \\ 
    \bottomrule
    \end{tabular}
}
    \caption{Complexity Comparison.}
    \label{tab:complexity}
\end{table}

\section{Discussion on Expressiveness of FMs, GNNs and \factorgnns}
\label{sec:expressiveness}
We envision one interesting branch of future work would be a unified framework of expressiveness for all three model categories: FMs, GNNs and \factorgnns.
To the best of our knowledge, there are currently two separate notions of expressiveness, one for FMs and the other for GNNs.
While these two notions of expressiveness are both widely acclaimed within their own communities, it is unclear how to bridge them and produce a new tool supporting the analysis of the empirical applications (\factorgnns) that seam the two communities.

\paragraph{Fully Expressiveness for Adjacency Recovery.}In the FM community, a FM is said to be \textit{fully expressive}~\cite{kazemi2018simple} if, for any given graph $\mathcal{T}$ over entities $\mathcal{E}$ and relations $\mathcal{R}$ , it can fully reconstruct the input adjacency tensor with a embedding size bounded by  $\min(|\mathcal{E}||\mathcal{R}|, |\mathcal{T}| + 1)$. 
We can generalise this expressiveness analysis to the spectrum of FM-GNN models (\factorgnns). In the $L \to \infty$ limit, \factorgnns are as fully expressive as the underlying FMs. In fact, a \factorgnn based on DistMult~\cite{yang2014embedding} is not fully expressive (because of its symmetry); however a \factorgnn based, e.g. on ComplEx~\cite{pmlr-v48-trouillon16,lacroix2018canonical} can reach full expressiveness for $L \to \infty$.

\paragraph{Weisfeiler-Leman Tests for Nodes/Graphs Separation.} For GNNs, established results concern the separation power of induced representations in terms of Weisfeiler-Leman (WL) isomorphism tests~\cite{DBLP:conf/iclr/XuHLJ19,geerts2021expressiveness}. 
However, none of these results is directly applicable to our setting (e.g. they only consider one relationship). 
Nevertheless, if we consider our \factorgnns in a one-relationship, simple graph setting, following the formalism of~\cite{geerts2021expressiveness}, we note that the \textsc{ReFactor} Layer function cannot be written in Guarded Tensor Language since at each layer it computes a global term $n[v]$.
Moreover, \factorgnns only process information coming from two nodes at one time. 
These two facts imply that \factorgnns have a separation power upper bound comparable to the 1-WL test, i.e. comparable to 1-MPNN (not guarded). 

We are not aware of explicit connections between the two above notions of expressiveness. We think there is some possibility that we can bridge them, which itself will be a very interesting research direction, but would require a very substantial amount of additional work and presentation space and is thus beyond the scope of this paper. 

Alternatively, we can also increase the expressiveness of \factorgnns by adding more parameters to the message, aggregation and update operators. For example, introducing additional MLPs to transform the input node features or include non-linearity in the GNN update operator. This would be a natural way to increase the expressiveness of \factorgnns. 

Another method for increasing expressive power for link prediction task only is to extend ReFactor GNNs from node-wise to pairwise (Sec 2.2 in our paper) representations like GraIL~\cite{Teru2020InductiveRP} and NBFNet~\cite{zhaochengzhu2021}, which is more computationally intensive, but yields more powerful as node representations are not standalone but adapted to a specific query.

\section{Experimental Details: Setup, Hyper-Parameters, and Implementation}
As we stated in the experiments section, we used a two-stage training process. In stage one, we sample subgraphs around query links and serialise them. In stage two, we load the serialised subgraphs and train the GNNs. For transductive knowledge graph completion, we test the model on the same graph~(but different splits). For inductive knowledge graph completion, we test the model on the new graph, where the relation vocabulary is shared with the training graph, while the entities are novel. We use the validation split for selecting the best hyper-parameter configuration and report the corresponding test performance. We include reciprocal triplets into the training triplets following standard practice~\cite{lacroix2018canonical}. 

For subgraph serialisation, we first sample a mini-batch of triplets and then use these nodes as seed nodes for sampling subgraphs. We also randomly draw a node globally and add it to the seed nodes. The training batch size is 256 while the valid/test batch size is 8.
We use the LADIES algorithm~\cite{ladies2019} and sample subgraphs with depths in $[1, 2, 3, 6, 9]$ and a width of 256. For each graph, we keep sampling for 20 epochs, i.e. roughly 20 full passes over the graph.

For general model training, we consider hyper-parameters including learning rates in $[0.01, 0.001]$, weight decay values in $[0, 0.1, 0.01]$, and dropout values in $[0, 0.5]$. For GATs, we use 768 as the hidden size and 8 as the number of attention heads. We train GATs with 3 layers and 6 layers. We also consider whether or not to combine the outputs from all the layers. For \factorgnns, we use the same hidden size as GAT. We consider whether the ReFactor Layer is induced by a SGD operator or by a AdaGrad operator. Within a ReFactor Layer, we also consider the N3 regulariser strength values $[0, 0.005, 0.0005]$, the $\alpha$ values $[0.1, 0.01]$, and the option of removing the $n[v]$, where the message-passing layer only involves information flow within 1-hop neighbourhood as most the classic message-passing GNNs do.

We use grid search to find the best hyper-parameter configuration based on the validation MRR.
Each training run is done using two Tesla V100~(16GB) GPUs with, where data parallelism was implemented via the \textit{DistributedDataParallel} component of \emph{pytorch-lightning}.
For inductive learning experiments, inference for all the validation and test queries on small datasets like FB15K237\_v1 takes about 1-5 seconds, while on medium datasets it takes approximately 20 seconds, and on big datasets like WN18RR\_v4 it requires approximately 60 seconds.
For most training runs, the memory footprint is less than 40\%~(13GB).
The training time for 20 full passes over the graph is  about 1, 7, and 21 minutes respectively for small, medium, and large datasets. 

Our code will be available at \href{https://github.com/yihong-chen/ReFactorGNN}{ReFactorGNN}.
We adapted the LADIES subgraph sampler from the GPT-GNN \href{https://github.com/UCLA-DM/GPT-GNN}{codebase}~\cite{gpt_gnn} for sampling on knowledge graphs. 
The datasets we used can be downloaded from the repositories \href{https://github.com/villmow/datasets_knowledge_embedding}{Datasets for Knowledge Graph Completion with Textual Information about Entities} and \href{https://github.com/kkteru/grail}{GraIL - Graph Inductive Learning}.
We implemented \factorgnns using the \href{https://pytorch-geometric.readthedocs.io/en/latest/notes/create_gnn.html}{\textit{MessagePassing}} API in \textit{PyTorch Geometric}. 
Specially, we used \href{https://pytorch-geometric.readthedocs.io/en/latest/notes/sparse_tensor.html}{\textit{message\_and\_aggregate}} function to compute the aggregated messages. Additionally, we note that, a much simplified operator, active forgetting~\citep{NEURIPS2023_6450ea28} can be an alternative way for implementing \factorgnns. Going beyond the discusssion betwen FMs and GNNs, \factorgnns especially in their active forgetting implementation can be seen as a meta-learning algorithm, where the inner level consists of the learning of entity embeddings while the outer loop takes care of the relation embeddings.

\section{Potential Negative Societal Impact}
Our work focus on efficient reasoning over knowledge graphs. A potential negative societal impact is that some people might use the methods for inferring private information using their own collected knowledge graphs.
However, this issue is also commonly faced by any other research work on knowledge graph reasoning.
\end{document}